\documentclass{article}
\PassOptionsToPackage{nonamebreak}{natbib}
\usepackage{iclr2027_conference,times}

\usepackage{amsmath,amsfonts,bm}

\def\eqref#1{equation~\ref{#1}}

\def\1{\bm{1}}

\DeclareMathAlphabet{\mathsfit}{\encodingdefault}{\sfdefault}{m}{sl}
\SetMathAlphabet{\mathsfit}{bold}{\encodingdefault}{\sfdefault}{bx}{n}

\usepackage{graphicx}
\usepackage{booktabs}
\usepackage{array}
\usepackage{amssymb}
\usepackage{amsthm}
\usepackage{algorithm}
\usepackage{algpseudocode}
\usepackage{xurl}
\usepackage{xcolor}
\usepackage{colortbl}
\usepackage{float}
\usepackage{placeins}
\usepackage{needspace}
\usepackage{hyperref}
\usepackage{url}
\definecolor{citationgreen}{RGB}{0,255,0}
\hypersetup{
  colorlinks=false,
  pdfborder={0 0 1},
  citebordercolor=citationgreen,
  linkbordercolor=red,
  urlbordercolor=white,
  filebordercolor=white,
  pdftitle={L2G-SR: Learning to Guide LLM-based Symbolic Regression}
}

\title{L2G-SR: Learning to Guide\\LLM-based Symbolic Regression}
\iclrfinalcopy
\author{Letian Huang \quad Xiangyu Wang \quad Haowei Lin \quad Di He\\
Peking University\\
{\small\normalfont \href{mailto:lotte_steven@stu.pku.edu.cn}{\texttt{\detokenize{lotte_steven@stu.pku.edu.cn}}}\quad \href{mailto:di_he@pku.edu.cn}{\texttt{\detokenize{di_he@pku.edu.cn}}}}}
\hypersetup{pdfauthor={Letian Huang, Xiangyu Wang, Haowei Lin, Di He}}

\hypersetup{hidelinks}
\newcommand{\dstr}{d_{\mathrm{str}}}
\newcommand{\vpred}{V_\theta}
\newcommand{\acc}{\mathrm{Acc}_{0.01}}
\newcommand{\nmse}{\mathrm{NMSE}}
\newcommand{\mse}{\mathrm{MSE}}
\newcommand{\ours}{L2G-SR}
\newtheorem{proposition}{Proposition}
\graphicspath{{figures/}}

\begin{document}
\maketitle
\fancyhead{}
\begin{abstract}
Symbolic regression searches for an equation that explains observed data. Numerical error is commonly used to guide the search, but it alone can favor an accurate surrogate over a candidate with the right structure. We propose \ours{}, a framework for LLM-based symbolic regression that guides search with a learned structural verifier. We construct structural supervision signal using tree edit distance (TED), accounting for commutative operators and alternative expression forms. A verifier first aligns numerical observations with the formulas, then learns to score a candidate's structural similarity to the target formula. At search time, the target formula is hidden and the verifier complements numerical error in parent selection. An analyst then uses verifier feedback and diagnostic tools to propose revisions, which a generator converts into refined expressions. On LLM-SRBench, \ours{} achieves state-of-the-art numerical fit, equation level accuracy and TED across the evaluated LLM backbones and domains. \ours{} demonstrates the value of learned structural guidance for LLM-based symbolic regression.
\end{abstract}

\section{Introduction}
\label{sec:intro}
Symbolic regression discovers both the structure and constants of an expression that explains observed variables, making it useful for interpretable modeling and the discovery of candidate scientific laws \citep{schmidt,pysr}. However, finding an optimal expression is NP-hard in general \citep{srhard}, and the combinatorial growth of possible structures makes efficient exploration essential. Existing approaches include search-based methods that evolve candidate expressions \citep{koza,pysr}, neural methods that learn to map observations to symbolic expressions \citep{biggio,e2e}, and methods based on large language models (LLMs) that combine scientific priors with numerical feedback to propose and refine equations \citep{llmsr,lasr,drsr,de,restart}.

Despite these advances, existing methods still rely heavily on numerical fit to guide search and evaluate recovered formulas using metrics such as normalized mean squared error (NMSE) and accuracy at a 1\% relative-error tolerance ($\acc$). It is essential feedback, but it need not identify the most useful structure to revise: near the origin, $x-x^3/6$ approximates $\sin x$ despite using different operators, while $\exp(1.05x)$ shares the form of $\exp(x)$ yet can accumulate substantial error. Numerical metrics alone are therefore insufficient to reliably guide formula discovery, motivating a complementary structural score.

To address this gap, we introduce \ours{} through three connected components, which guides LLM-based SR with a learned structural verifier (Figure~\ref{fig:overview}). A \emph{modified tree edit distance} (TED) abstracts constants, matches commutative operands, and compares selected algebraic views to label structural similarity. A \emph{learned verifier} first aligns observations with their generating formulas, then learns to predict a candidate's structural similarity. A \emph{two-signal search} combines verifier output and numerical error to select parents, suggest edits, and retain candidates.

On 240 LLM-SRBench equations, \ours{} achieves the highest or the second highest performance on all metrics and LLM backbones, showing the strong performance of our method. Replacing its structural channel with validation MSE substantially lowers the performance.

\newpage 
\begin{figure}[H]
    \centering
    \includegraphics[width=\linewidth]{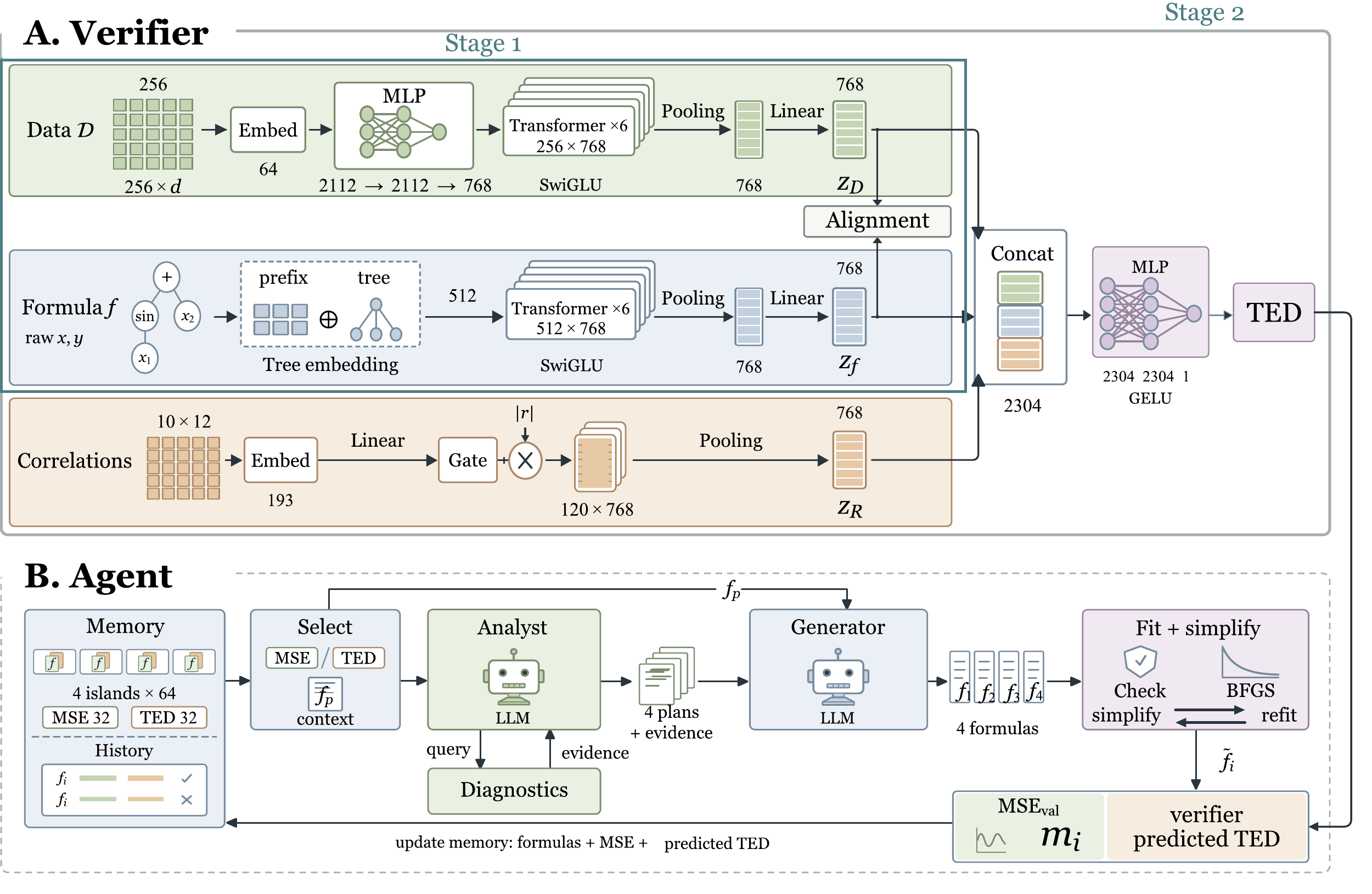}
    \caption{Overview of \ours{}. (A) A verifier learns TED with two-stage training. (B) During search, numerical error and predicted TED guide candidate selection and retention.}
    \label{fig:overview}
\end{figure}

Our contributions are threefold:
\begin{itemize}
    \item We propose a modified TED that measures structural similarity between formulas while accommodating common simplifications and equivalent forms, providing structural supervision beyond numerical fit alone.
    \item We develop a multimodal TED verifier that predicts a candidate's structural distance to an unseen target formula from the candidate and its observations, making structural feedback available during symbolic regression.
    \item We integrate this learned feedback into an LLM-guided search agent, enabling it to prioritize and revise formulas that are structurally promising as well as numerically accurate.
\end{itemize}

\section{Related work}
\label{sec:related}

\paragraph{Symbolic and neural search.}
Genetic programming, Eureqa, Operon, PySR, GP-GOMEA, AI Feynman, and Bayesian SR search expressions through evolution, physical structure, or probabilistic modeling~\citep{koza,schmidt,operon,pysr,gpgomea,aifeynman,aifeynman2,bsr}; SRBench compares representative methods~\citep{srbench}. DSR and uDSR learn generation policies, while neural-guided GP, DGSR-MCTS, and TPSR combine learned proposals with search~\citep{dsr,unifieddsr,mundhenk,dgsrmcts,tpsr}. Pretrained and end-to-end models amortize discovery over synthetic corpora~\citep{biggio,e2e,dgsr}. SNIP, MDLformer, and MOD-SR connect numerical and symbolic representations~\citep{snip,mdlformer,modsr}. We use such representations to score a candidate's structural similarity for an external search procedure.

\paragraph{LLM-guided discovery and structural comparison.}
LLM-SR, LaSR, and DrSR use language-model priors to refine expressions~\citep{llmsr,lasr,drsr}. DE, RESTART, SR-Scientist, and related systems add diagnostics, reusable structures, multiple roles, or engine control~\citep{de,restart,srscientist,asr,guided,llmpysr,sga}. Our distinction is the supervised structural signal in selection, diagnosis, and retention. TED compares labeled trees~\citep{ted}; SRSD and PHYBench evaluate recovered structure, and EGG-SR represents equivalences with equality graphs~\citep{srsd,phybench,eggsr}.

\section{Method}
\label{sec:method}
\subsection{Problem formulation}
\label{sec:problem}

Given observations $\mathcal D=\{(x_i,y_i)\}_{i=1}^n$, where $x_i\in\mathbb R^p$ and $y_i=f^\star(x_i)+\epsilon_i$, symbolic regression seeks a compact expression $f$ for the unknown target $f^\star$. The admissible expression space $\mathcal F$ is defined by the available variables, operators, and constants; each candidate has a symbolic structure and fitted numerical constants. Numerical fit is commonly measured by mean squared error,
\begingroup
\setlength{\abovedisplayskip}{3pt}
\setlength{\belowdisplayskip}{3pt}
\setlength{\abovedisplayshortskip}{2pt}
\setlength{\belowdisplayshortskip}{2pt}
\begin{equation}
    \mathcal L_{\mse}(f;\mathcal D)=\frac{1}{n}\sum_{i=1}^{n}\bigl(f(x_i)-y_i\bigr)^2.
    \label{eq:sr_mse}
\end{equation}
\endgroup
Finite observations can support expressions with different structures, so low empirical error alone does not establish similarity to $f^\star$. We therefore seek candidates with good numerical fit while estimating their structural similarity to the target from $\mathcal D$ and the candidate expression. The target expression remains hidden during search.

\subsection{Structural similarity measure}
\label{sec:distance}
We use a modified tree edit distance (TED), $\dstr(f,g)$, to compare formula structures despite fitted constants, operand order, and common algebraic rewrites; smaller values indicate closer structures.

\noindent\textbf{Step 1: Abstract trees and edit costs.}
To ignore constant values and grouping, we map numeric-only subtrees to \texttt{number}, symbolic constants to \texttt{param}, and flatten Add/Mul nodes. Insertions, deletions, and other relabelings cost one; identical labels and \texttt{number}/\texttt{param} exchanges cost zero \citep{ted}. With PHYBench's discount \citep{phybench}, insertion or deletion of a subtree of size $s$ costs

\begin{equation}
    \kappa(s)=\min\{s,\;0.6(s-5)+5\}.
    \label{eq:subtree}
\end{equation}

\noindent\textbf{Step 2: Match commutative operands.}
Let $D_{\mathrm{ord}}(A,B)$ be ordered TED under these costs. At addition or multiplication roots, we also match children by minimum-cost assignment to reduce sensitivity to operand order:
\begin{equation}
 D_{\mathrm{hyb}}(A,B)=
 \begin{cases}
 \min\!\left\{D_{\mathrm{ord}}(A,B),\,
 c(a,b)+\min_{\pi}\sum_{i=1}^{m} C_{i,\pi(i)}\right\},
 &a,b\in\{\mathrm{Add},\mathrm{Mul}\},\\
 D_{\mathrm{ord}}(A,B),&\text{otherwise}.
 \end{cases}
 \label{eq:hybrid}
\end{equation}
Here $a,b$ are root labels, $c$ is their relabeling cost, and $m$ is the larger child count. $C_{ij}$ is the recursive hybrid cost for real children or $\kappa(s)$ for a dummy match. We retain the cheaper of ordered editing and child assignment.

\noindent\textbf{Step 3: Compare algebraic views.}
We compare three views to tolerate common rewrites: a structure-preserving candidate parse (evaluated reference), ordinary evaluation with constant absorption, and expansion with power simplification. Let $P_k^c,P_k^r$ transform the candidate and reference, and $T$ build the tree:
\begin{equation}
 \dstr(f,g)=
 \min_{k\in\{0,1,2\}}
 \frac{D_{\mathrm{hyb}}\!\left(T(P_k^c(f)),T(P_k^r(g))\right)}
 {|T(P_k^r(g))|}.
 \label{eq:target}
\end{equation}
Each cost is relative to the reference-tree size. Appendix~\ref{app:distance} specifies the transformations.

Figure~\ref{fig:abstract_tree_invariances} illustrates the shared abstract representation and the distinctions retained by this measure. More examples and explanations are provided in Appendix~\ref{app:distance}.

\begin{figure}[H]
    \centering
    \includegraphics[width=\linewidth]{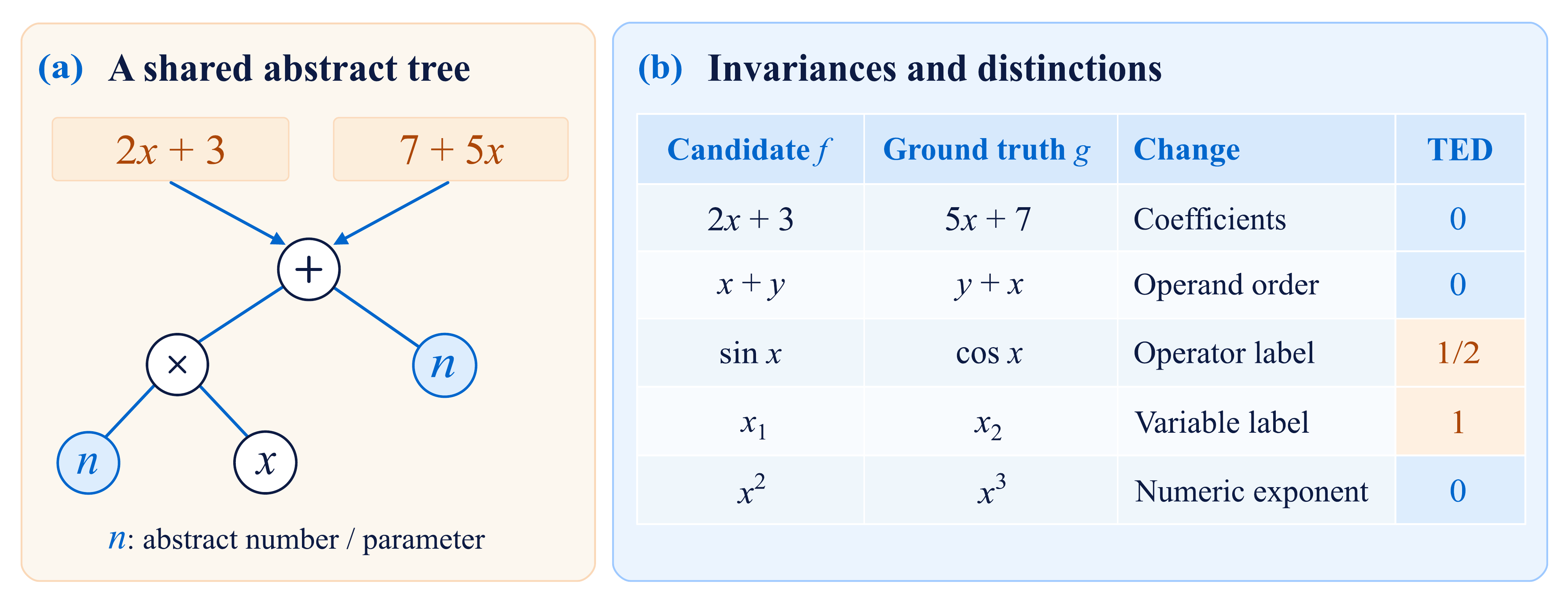}
    \caption{Examples of structural similarity under the modified TED. (a) The two linear expressions map to the same abstract tree. (b) Scores are normalized by the ground-truth tree size. In these examples, changes in coefficient values, operand order, and numeric exponent values yield zero distance, whereas operator and variable label changes incur a cost.}
    \label{fig:abstract_tree_invariances}
\end{figure}

\subsection{Data generation}
\label{sec:data_generation}

Training our Transformer verifier requires a large volume of paired numerical and symbolic data. To produce diverse synthetic formulas that resemble scientific equations and search-time expressions, we use a fixed data-generation scheme informed by SRBench/Feynman ground-truth formulas and recorded LLM-SR/PySR search trajectories on SRBench. These sources provide realistic structural priors, while all training instances are independently generated by our own synthesis algorithm, without access to LLMSRbench. The resulting data pools contain 91.9 million target--candidate pairs. To prevent benchmark overlap, we remove training pairs that share the same symbolic structure with any LLM-SRBench ground-truth expression after constant abstraction and symbolic canonicalization. More details are provided in Appendix~\ref{app:generator}.

\noindent\textbf{Random formulas.}
We first sample the total number of nodes and operators from empirical priors. We then sample an ordered tree shape using Catalan weights and place the operators, variables, and constants. These priors balance ground-truth and LLM-SR search formulas, preserving common structures while generating diverse new expressions.

\noindent\textbf{Target--candidate pairs.}
We construct two types of target--candidate pairs: random and edited. Random pairs draw a candidate $f$ independently of the target $f^\star$; edited pairs obtain $f$ by locally modifying $f^\star$, including point relabeling, unary insertion or deletion, subtree replacement, etc. Random pairs mainly generate distant pairs (TED $\geq 0.6$), whereas edited pairs mainly generate close pairs (TED $<0.6$). Our aim is to cover as broad a range of TED values as possible. To reduce bias from the uneven TED distribution of generated pairs, we sample 12 TED intervals equally.

\noindent\textbf{Simplification and validation.}
We fold constant-only subtrees and apply selected algebraic simplifications. We reject scientifically implausible formulas, such as those with atypical operator nesting or invalid domains. This filtering aims to keep synthetic formulas scientifically plausible and closer to the structural patterns of LLM-generated search expressions.

\noindent\textbf{Numerical observations.}
We adapt the observation sampler of MDLformer \citep{mdlformer} to generate observation points that cover diverse regions of the input space. For each synthetic target $f^\star$, we draw inputs from a mixture of randomly located and rotated Gaussian or uniform components. The resulting observation set is
\begin{equation}
 \begin{aligned}
 \mathcal D^\star&=\{(\mathbf x_i,f^\star(\mathbf x_i))\}_{i=1}^{M},
 &\mathbf x_i&\sim\sum_{k=1}^{C}\rho_k Q_k,
 &Q_k&\in\{Q_k^{\mathrm G},Q_k^{\mathrm U}\},\\
 Q_k^{\mathrm G}&=\operatorname{Law}\!\left((\boldsymbol\mu_k+\boldsymbol\epsilon_k)R_k\right),
 &\boldsymbol\epsilon_k&\sim\mathcal N\!\left(\mathbf 0,\operatorname{diag}(\mathbf c_k)\right),\\
 Q_k^{\mathrm U}&=\operatorname{Law}\!\left((\boldsymbol\mu_k+\sqrt{\mathbf c_k}\odot\mathbf u_k)R_k\right),
 &\mathbf u_k&\sim\mathcal U([-1,1]^d).
 \end{aligned}
 \label{eq:generated_observations}
\end{equation}
Here, $R_k$ is a random orthogonal rotation; before rotation, the mean is $\boldsymbol\mu_k$ and coordinatewise variances are $\mathbf c_k$ (Gaussian) or $\mathbf c_k/3$ (uniform).

\subsection{Verifier}
\label{sec:learning}
\noindent\textbf{Architecture.}

The verifier first aligns numerical observations $\mathcal D$ with their generating formulas in a shared embedding space. It then combines the observation embedding with a candidate formula embedding and regression summaries $R(\mathcal D)$ to estimate the candidate's TED to the unknown target:
\begin{equation}
 z_D=E_D(\mathcal D),\qquad
 z_f=E_f(f),\qquad
 z_R=E_R(R(\mathcal D)),\qquad
 \widehat d=h_\theta([z_D;z_f;z_R]).
 \label{eq:architecture}
\end{equation}
\noindent\textbf{Embedding.}
Following MDLformer \citep{mdlformer}, we encode each input and target value in scientific notation using three tokens: its sign (positive or negative), its mantissa (the coefficient rounded to four significant digits), and its base-10 exponent (which specifies the power of ten multiplying the coefficient); for example, $54.321\approx 5.432\times10^{1}\quad\longmapsto\quad[+,\,5.432,\,+1]$.
As in MDLformer \citep{mdlformer}, we tokenize $f$ in prefix order. However, TED prediction relies on the expression's tree structure, while prefix order does not explicitly mark each token's depth or branch. We therefore add a tree address to each token \citep{mathgpt}. For a root-to-node path with child slots $(c_1,\ldots,c_d)$, its address contribution is
\begin{equation}
 a=W_A\frac{\sum_{\ell=1}^{d}e_{\mathrm{level}}(\ell)\odot e_{\mathrm{slot}}(c_\ell)}{\sqrt{\max(d,1)}}+e_{\mathrm{depth}}(d).
 \label{eq:address}
\end{equation}
At each depth, the level and child-slot embeddings together record the branch taken; adding these records describes the path from root to node, scaling limits the effect of path length, and $W_A$ places the result in the token-embedding space while $e_{\mathrm{depth}}(d)$ marks its depth.

\noindent\textbf{Regression summaries.}
For each variable, we evaluate 12 simple relations, including $\log y\sim\log x_j$, $y\sim x_j$, and $y\sim x_j^2$ (Appendix~\ref{app:training}). Let $r$ be the Pearson correlation between the transformed input and target values. Each valid relation is encoded as
\[
\mathbf t=|r|\,\operatorname{Linear}([\mathbf e_{\mathrm{relation}};r]).
\]
Here $\mathbf e_{\mathrm{relation}}$ joins three learned vectors that specify whether the output is $y$ or $\log y$, the operation on the input (such as $x$, $x^2$, or $\log x$), and which input variable $x_j$ is used. The $|r|$ factor downweights weak correlations. The log--log relation probes how input and output magnitudes scale together; the other 11 probe different input--output relations, such as linear and quadratic patterns.

\noindent\textbf{Two-Stage Training.}

\noindent\textbf{Stage 1: symbolic--numeric alignment.}
Following MDLformer \citep{mdlformer}, which builds on SNIP \citep{snip}, we use symmetric in-batch contrastive learning to align $u_i=E_D(\mathcal D_i)$ with $v_i=E_f(f_i^\star)$, so observations and their generating formulas have consistent embeddings:
\begin{equation}
 \mathcal L_{\mathrm{align}}=-\frac{1}{2B}\sum_{i=1}^{B}
 \left[\log\frac{\exp(u_i^\top v_i/t)}{\sum_{j=1}^{B}\exp(u_i^\top v_j/t)}
 +\log\frac{\exp(v_i^\top u_i/t)}{\sum_{j=1}^{B}\exp(v_i^\top u_j/t)}\right],
 \label{eq:align_loss}
\end{equation}
\noindent\textbf{Stage 2: structural prediction.}
Starting from the Stage~1 encoders, we train the verifier to predict normalized TED $d=\dstr(f,f^\star)$ for observation--candidate pairs using
\begin{equation}
 \mathcal L(\widehat d,d)=
 \begin{cases}
 |\widehat d-d|,& d\leq 1,\\
 \max(0,1-\widehat d),& d>1.
 \end{cases}
 \label{eq:loss}
\end{equation}
For $d>1$, candidates are already structurally far from the target, so we only require $\widehat d\geq1$. This limits the effect of large distances and helps keep predictions from clustering near the mean. Appendix~\ref{app:training} gives the training configuration and checkpoint selection.

\subsection{L2G-SR: TED-guided symbolic search}
\label{sec:search}
We present L2G-SR, a symbolic regression method that finds formulas by repeatedly proposing and testing candidates \citep{pysr,llmsr}. Its core principle is to guide the search with two complementary signals: numerical fit to the observations and predicted TED to the unknown target formula. An LLM analyst suggests revisions, and an LLM generator writes new formulas. Further details are provided in Appendix~\ref{app:search}.

We compute validation MSE on a fixed 20\% holdout of the search observations to measure each candidate's predictive accuracy beyond the fitting subset.

We begin by generating formulas for four islands, each a separate candidate pool. In each search round, we select an island and a parent formula using either validation MSE or predicted TED. The analyst examines the parent and suggests four distinct ways to revise it; the generator then turns these suggestions into four complete formulas. We check each formula, fit its constants, simplify it, and score it with both signals before updating the selected island. After the search, we choose a simple formula among those with the best numerical fit. Algorithm~\ref{alg:search} summarizes the loop.

\noindent\textbf{Dual-score selection.}
Each round considers one of the four islands. We randomly draw a few formulas from that island, then choose one to revise (the parent): half the rounds take the formula with the lowest validation MSE, and the other half take the one with the lowest predicted TED. If $\mathcal S$ denotes the formulas drawn in the current round, then
\begingroup
\setlength{\abovedisplayskip}{3pt}
\setlength{\belowdisplayskip}{3pt}
\setlength{\abovedisplayshortskip}{2pt}
\setlength{\belowdisplayshortskip}{2pt}
\begin{equation}
 f_{\mathrm{parent}}=\begin{cases}
 \arg\min_{f\in\mathcal S}\mse_{\mathrm{val}}(f), & \text{MSE round},\\
 \arg\min_{f\in\mathcal S}\widehat d(f), & \text{TED round}.
 \end{cases}
\end{equation}
\endgroup

\noindent\textbf{Analyst agent.}
The analyst receives validation MSE and predicted TED for the parent and several retained formulas, plus summaries of the parent's residuals and transformed-variable regressions. In TED rounds, it compares lower-TED formulas with the parent to spot terms or operators worth testing. In MSE rounds, it focuses on patterns in the residual $y-f(x)$. To check a specific idea, it can name a reference formula in a structure-comparison request or a parent subtree in a residual-probe request. The first tool reports which parts differ; the second tests whether the selected subtree explains the residual on training data. Using these results and past edit outcomes, the analyst gives the generator four specific suggestions about what to change or keep.

\noindent\textbf{Generator agent.}
A second LLM, the generator, uses the parent, the analyst's suggestions, and cited evidence to produce four complete formulas. Each formula occupies one proposal slot and is checked, fitted, and scored before it can be kept for later rounds. The agent is told the physical units and restrictions needed for formulas to remain real-valued.

\noindent\textbf{Fitting and simplification.}
For each generated formula, we keep its symbolic structure fixed and use BFGS to optimize any unknown constants on the fitting subset.

We use SymPy to simplify algebraically equivalent expressions and combine redundant constants. We also remove terms that contribute little to the formula's predictions, as defined by Equation~\ref{eq:contribution} in Appendix~\ref{app:search}. Such terms create spurious branches in the formula tree, which later revisions may inherit or modify and thus steer the search away from the useful structure. They can also make the verifier predict an excessively large TED for an otherwise promising formula. We replace exponents that are already close to simple values. We optimize the remaining constants again with BFGS after removing terms or replacing exponents. The verifier scores only the final accepted formula.

\noindent\textbf{Memory.}
We keep four islands, each with at most 64 formulas, plus a separate history of proposals and outcomes for the analyst. After each round, valid formulas enter the selected island, and we compute each member's island-local rank by validation MSE and by predicted TED, with lower scores ranked first in both. When formulas share the same structure, only the version with lower validation MSE enters these rankings. If an island exceeds 64 formulas, we prune it using the two rankings. We keep up to 32 leading formulas from each ranking. The survivors can become parents in later rounds; final selection favors simple formulas among those with the best validation fit.

\section{Experiments}
\label{sec:experiments}

\begingroup
\raggedbottom
\makeatletter
\def\subsection{\@startsection{subsection}{2}{\z@}{-0.8ex}{0.8ex}{\normalsize\sc\raggedright}}
\makeatother

\subsection{Experimental setup}
\label{sec:experimental_setup}
\noindent\textbf{Benchmark.}
We evaluate on LLM-SRBench \citep{llmsrbench}, a symbolic regression benchmark designed to test discovery beyond memorized canonical formulas. Our evaluated configuration contains 240 equations across five domains: 111 Transform, 44 Physics, 25 Materials, 36 Chemistry, and 24 Biology.

\noindent\textbf{Baselines.}
We compare \ours{} with representative LLM-based symbolic regression methods, including LLMDirect \citep{llmsrbench}, LLM-SR \citep{llmsr}, LaSR \citep{lasr}, SGA \citep{sga}, DE \citep{de}, A-SR \citep{asr}, RESTART \citep{restart}, and SR-Scientist \citep{srscientist}.

\noindent\textbf{Metrics.}
\label{sec:evaluation_metrics}
We report normalized mean squared error (NMSE; lower is better), accuracy at a relative-error tolerance of 0.01 ($\acc$; higher is better), and true normalized tree edit distance (TED; lower is better) between a selected formula and the target expression. Following DE's evaluation protocol \citep{de}, we score test predictions in float64 and discard 5\% of the observations with the largest relative errors for each equation. For the retained set $I$ and its target mean $\bar y_I$, the numerical metrics are
\begingroup
\setlength{\abovedisplayskip}{3pt}
\setlength{\belowdisplayskip}{3pt}
\setlength{\abovedisplayshortskip}{3pt}
\setlength{\belowdisplayshortskip}{3pt}
\begin{equation}
 \nmse(f)=
 \frac{\sum_{i\in I}(f(x_i)-y_i)^2}
 {\sum_{i\in I}(y_i-\bar y_I)^2},\qquad
 A(f)=\mathbf1\!\left[
 \max_{i\in I}\frac{|f(x_i)-y_i|}{\max(|y_i|,10^{-12})}\leq0.01
 \right].
 \label{eq:metrics}
\end{equation}
\endgroup
We average equation-level NMSE, $A(f)$, and TED within each domain.

\subsection{Main experimental results}
\label{sec:main_results}

\begingroup
\setlength{\intextsep}{4pt}
\setlength{\abovecaptionskip}{4pt}
\begin{table}[H]
\caption{Performance on five LLM-SRBench domains with Llama3.1-8B-Instruct and Qwen3-4B-Instruct-2507. Columns report NMSE ($\downarrow$) and $\acc$ ($\uparrow$, \%). LLMDirect, LLM-SR, LaSR, SGA, and DE results are from DE \citep{de}; A-SR results are from \citet{asr}. \ours{} shows mean $\pm$ sample standard deviation over three seeds. Bold and underlined means mark the best and second-best displayed values.}
\label{tab:main}
\centering
\scriptsize
\setlength{\tabcolsep}{1.5pt}
\renewcommand{\arraystretch}{1.15}
\newcommand{\tablemeanstd}[2]{\begin{tabular}[t]{@{}c@{}}\strut#1\\{\tiny\strut\textpm#2}\end{tabular}}
\resizebox{\linewidth}{!}{%
\begin{tabular}{@{}l*{5}{cc}@{}}
\toprule
 & \multicolumn{2}{c}{\textbf{Transform}} & \multicolumn{2}{c}{\textbf{Physics}} & \multicolumn{2}{c}{\textbf{Materials}} & \multicolumn{2}{c}{\textbf{Chemistry}} & \multicolumn{2}{c}{\textbf{Biology}}\\
\cmidrule(lr){2-3}\cmidrule(lr){4-5}\cmidrule(lr){6-7}\cmidrule(lr){8-9}\cmidrule(lr){10-11}
\textbf{Method} & NMSE $\downarrow$ & $\acc$ $\uparrow$ & NMSE $\downarrow$ & $\acc$ $\uparrow$ & NMSE $\downarrow$ & $\acc$ $\uparrow$ & NMSE $\downarrow$ & $\acc$ $\uparrow$ & NMSE $\downarrow$ & $\acc$ $\uparrow$\\
\midrule
\rowcolor{gray!10}
\multicolumn{11}{c}{\textbf{Llama3.1-8B-Instruct}}\\
LLMDirect & 2.95e-1 & 34.23 & 9.95e-3 & 0.00 & 8.16e-2 & 24.00 & 1.22e0 & 2.78 & 1.29e-1 & 4.17\\
LLM-SR & 2.42e-1 & 34.23 & 3.00e-3 & 6.82 & 2.16e-1 & 60.00 & 5.24e-2 & 16.67 & 1.76e-2 & 12.50\\
LaSR & 2.62e-1 & 33.33 & 6.07e-3 & 9.09 & 9.47e-4 & 32.00 & 1.82e-3 & 8.33 & 6.40e-3 & 0.00\\
SGA & 3.52e-1 & 0.90 & 1.55e-1 & 2.27 & 4.35e-2 & 12.00 & 4.58e-2 & 8.33 & 2.42e-1 & 0.00\\
DE & \underline{1.12e-1} & 36.04 & 1.01e-3 & 11.36 & 2.89e-4 & 64.00 & 4.16e-4 & 11.11 & 1.17e-2 & 16.67\\
A-SR & \textbf{\boldmath 1.09e-1} & \underline{40.54} & \underline{9.03e-4} & \underline{27.27} & \underline{7.45e-6} & \underline{84.00} & 2.60e-5 & 31.94 & \textbf{\boldmath 5.38e-6} & \underline{50.00}\\
RESTART & 2.22e-1 & 4.80 & 7.48e-3 & 6.06 & 8.19e-5 & 66.67 & \underline{2.24e-6} & 60.19 & 1.80e-2 & 22.22\\
SR-Scientist & 3.03e-1 & 3.00 & 3.78e-3 & 23.48 & 1.77e-3 & 52.00 & 4.96e-5 & \underline{63.89} & 7.26e-2 & 20.83\\
\rowcolor{blue!4}
\ours & \tablemeanstd{\underline{1.12e-1}}{0.04e-1} & \tablemeanstd{\textbf{\boldmath 43.54}}{1.04} & \tablemeanstd{\textbf{\boldmath 2.70e-4}}{1.11e-4} & \tablemeanstd{\textbf{\boldmath 89.39}}{3.47} & \tablemeanstd{\textbf{\boldmath 1.07e-7}}{0.23e-7} & \tablemeanstd{\textbf{\boldmath 100.00}}{0.00} & \tablemeanstd{\textbf{\boldmath 1.23e-6}}{0.67e-6} & \tablemeanstd{\textbf{\boldmath 85.05}}{1.50} & \tablemeanstd{\underline{6.06e-3}}{4.05e-3} & \tablemeanstd{\textbf{\boldmath 68.06}}{6.36}\\
\midrule
\rowcolor{gray!10}
\multicolumn{11}{c}{\textbf{Qwen3-4B-Instruct-2507}}\\
LLMDirect & 3.55e-1 & 24.32 & 5.46e-2 & 6.82 & 1.42e-3 & 52.00 & 2.66e-1 & 2.78 & 4.46e-2 & 8.33\\
LLM-SR & 3.15e-1 & 26.13 & 2.51e-3 & 6.82 & 3.55e-3 & 44.00 & 3.36e-2 & 13.89 & 1.88e-2 & 12.50\\
LaSR & 1.83e-1 & 30.91 & 6.04e-3 & 6.82 & 6.21e-4 & 8.00 & 2.31e-3 & 2.78 & 9.56e-3 & 0.00\\
SGA & 4.09e-1 & 19.81 & 1.04e-1 & 0.00 & 1.02e-2 & 16.00 & 1.61e-1 & 2.78 & 1.73e-1 & 4.17\\
DE & \textbf{\boldmath 1.15e-1} & \textbf{\boldmath 50.45} & \underline{4.37e-4} & 15.91 & 1.47e-4 & 56.00 & 1.88e-4 & 13.89 & 6.69e-3 & 12.50\\
A-SR & 1.37e-1 & 33.33 & 6.68e-4 & \underline{21.59} & \underline{3.48e-5} & \underline{68.00} & 1.27e-5 & 34.72 & \textbf{\boldmath 1.25e-3} & 20.83\\
RESTART & 6.25e1 & 1.50 & 1.86e-1 & 3.03 & 5.92e-3 & 4.00 & 6.95e-2 & 24.07 & 1.78e0 & 9.72\\
SR-Scientist & 2.50e-1 & 19.82 & 2.90e-2 & 15.91 & 1.08e-3 & 64.00 & \underline{6.50e-6} & \underline{61.11} & 9.44e-2 & \underline{41.67}\\
\rowcolor{blue!4}
\ours & \tablemeanstd{\underline{1.18e-1}}{0.08e-1} & \tablemeanstd{\underline{42.94}}{0.52} & \tablemeanstd{\textbf{\boldmath 3.49e-4}}{1.12e-4} & \tablemeanstd{\textbf{\boldmath 87.88}}{2.62} & \tablemeanstd{\textbf{\boldmath 1.20e-7}}{0.15e-7} & \tablemeanstd{\textbf{\boldmath 98.67}}{2.31} & \tablemeanstd{\textbf{\boldmath 1.34e-6}}{0.18e-6} & \tablemeanstd{\textbf{\boldmath 79.63}}{4.24} & \tablemeanstd{\underline{3.97e-3}}{1.28e-3} & \tablemeanstd{\textbf{\boldmath 61.11}}{2.41}\\
\bottomrule
\end{tabular}%
}
\end{table}

\endgroup

\noindent\textbf{Research Question 1: How accurately does L2G-SR fit benchmark equations?}
Across both backbones and all five LLM-SRBench domains, \ours{} achieves state-of-the-art performance in numerical fit and equation-level accuracy (Table~\ref{tab:main}). Its NMSE and $\acc$ rank first or second for every backbone and domain, showing consistent performance across models and domains. In the available final-formula comparisons, \ours{} also attains the lowest mean true TED in every domain under both backbones (Figure~\ref{fig:final_formula_ted}), indicating that it finds formulas structurally closer to the targets. Additional RQ1 results, including baseline variability and LLM usage, appear in Appendix~\ref{app:results} (Tables~\ref{tab:baseline_sd} and~\ref{tab:llm_usage}).
\par

\begingroup
\setlength{\intextsep}{4pt}
\begin{figure}[H]
\centering
\includegraphics[width=\linewidth]{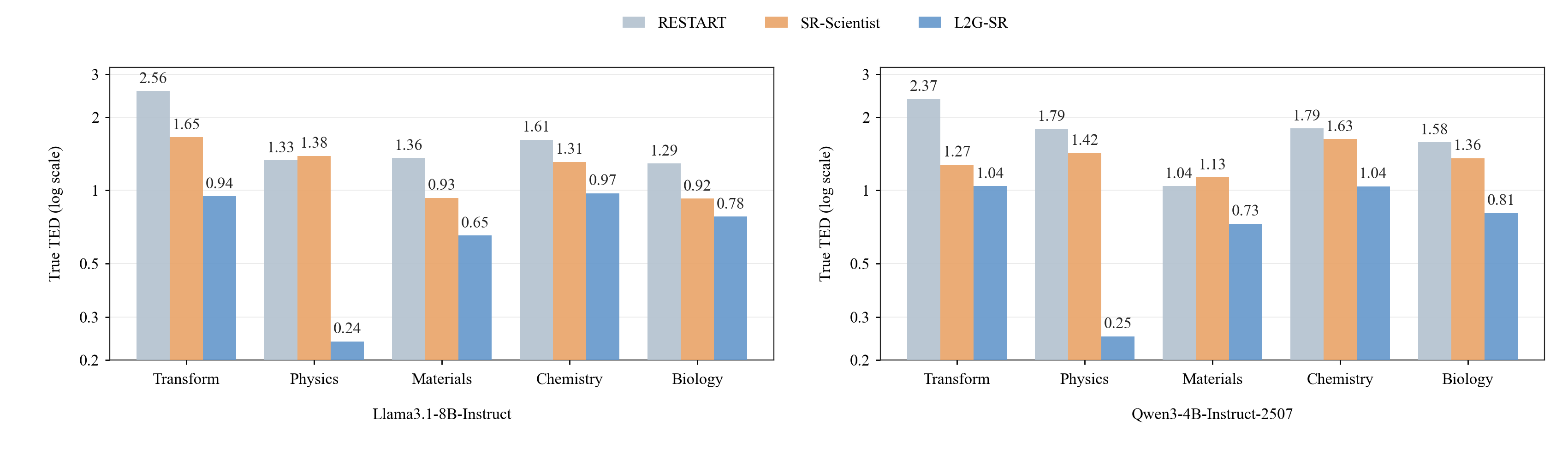}
\caption{True TED ($\downarrow$) of the final formulas from Table~\ref{tab:main}.}
\label{fig:final_formula_ted}
\end{figure}
\endgroup

\subsection{Verifier analysis}
\label{sec:verifier_analysis}
\noindent\textbf{Research Question 2: How well does the verifier perform?}\par\nopagebreak[4]
\noindent\textbf{Stage~2 random test set.} Table~\ref{tab:verifier_stage2} reports metrics on the random test set for the checkpoint used in search. The checkpoint performs strongly on this random dataset, with accurate estimates for structurally close formulas and useful detection of more distant ones.

\begingroup
\setlength{\intextsep}{2pt}
\begin{table}[H]
\begin{minipage}[t]{0.49\linewidth}
\vspace{0pt}
\centering
\small
\setlength{\tabcolsep}{4pt}
\begin{tabular}{@{}lc@{}}
\toprule
Metric & Selected\\
\midrule
Val. objective $\downarrow$ & 0.12613\\
Val. MAE$_{\leq1}$ $\downarrow$ & 0.13086\\
Val. recall$_{>1}$ $\uparrow$ (\%) & 52.63\\
\bottomrule
\end{tabular}
\caption{\raggedright Selected Stage~2 checkpoint on the synthetic validation set. Objective follows Eq.~(\ref{eq:loss}). MAE$_{\leq1}$ is the mean absolute error on targets with TED $\leq1$; Recall$_{>1}$ is the percentage of targets with TED $>1$ predicted at or above 1.}
\label{tab:verifier_stage2}

\end{minipage}\hfill
\begin{minipage}[t]{0.49\linewidth}
\vspace{0pt}
\centering
\small
\setlength{\tabcolsep}{4pt}
\begin{tabular}{@{}lcc@{}}
\toprule
Domain & MSE $\downarrow$ & Correlation\\
\midrule
Transform & $0.4550$ & $0.3836$\\
Physics & $0.1107$ & $0.3310$\\
Materials & $0.1148$ & $0.2505$\\
Chemistry & $0.1393$ & $0.5808$\\
Biology & $0.1000$ & $0.4318$\\
\bottomrule
\end{tabular}
\caption{\raggedright Verifier performance measured after search on recorded candidate formulas. Correlation is Pearson $r$.}
\label{tab:verifier_summary}

\end{minipage}
\end{table}
\endgroup

\noindent\textbf{Downstream predictions.} On recorded search candidates, the verifier predicts TED with reasonable accuracy (Table~\ref{tab:verifier_summary}), providing effective structural guidance that improves downstream search (Table~\ref{tab:ablation}).

\Needspace{5\baselineskip}
The matched Physics ablation gives more direct evidence for the verifier's role: replacing its score with validation MSE in the structural channel worsens test NMSE and accuracy while increasing the final formula's true TED (Table~\ref{tab:ablation}).

\subsection{Ablation study}
\label{sec:ablation}
\noindent\textbf{Research Question 3: Does TED feedback improve search under a matched budget?} We ablate the TED verifier on Physics by replacing its predicted TED score with internal-validation MSE throughout the structural channel. We keep two-role generation, diagnostics, fitting and simplification, the selection schedule, retention policy, and the 400-slot budget fixed, so the comparison isolates the verifier's structural guidance. As shown in Table~\ref{tab:ablation}, removing this guidance substantially worsens both NMSE and accuracy, demonstrating that the TED verifier is an important contributor to \ours{}'s search performance in this setting.

Notably, this downstream benefit is observed despite moderate and domain-dependent candidate-level verifier correlations (Table~\ref{tab:verifier_summary}). Together, Tables~\ref{tab:verifier_summary} and~\ref{tab:ablation} suggest that the verifier's value may lie less in precisely estimating absolute TED for every candidate than in providing sufficiently informative relative structural preferences during search.

\begingroup
\setlength{\intextsep}{4pt}
\begin{table}[t]
\begin{minipage}[t]{0.53\linewidth}
\vspace{0pt}
\centering
\footnotesize
\setlength{\tabcolsep}{1pt}
\begin{tabular}{@{}lcc@{}}
\toprule
Metric & MSE / MSE & MSE / TED\\
\midrule
NMSE $\downarrow$ & $(2.264\pm2.590)\mathrm{e}{-3}$ & $(2.702\pm1.109)\mathrm{e}{-4}$\\
$\acc$ $\uparrow$ (\%) & $74.24\pm10.74$ & $89.39\pm3.47$\\
True TED $\downarrow$ & $0.361\pm0.112$ & $0.239\pm0.010$\\
\bottomrule
\end{tabular}
\caption{\raggedright Matched Physics ablation: mean $\pm$ sample standard deviation across three seeds, with 44 tasks per seed. The numerical-only condition replaces TED by MSE throughout the structural-score path.}
\label{tab:ablation}
\end{minipage}\hfill
\begin{minipage}[t]{0.45\linewidth}
\vspace{0pt}
\centering
\scriptsize
\setlength{\tabcolsep}{2.5pt}
\begin{tabular}{@{}lccc@{}}
\toprule
Verifier & Objective $\downarrow$ & MAE$_{\leq1}$ $\downarrow$ & Recall$_{>1}$ $\uparrow$\\
\midrule
w/o tree address & 1.62e-1 & 1.79e-1 & \textbf{52.63}\\
Formula only & 2.29e-1 & 2.33e-1 & 36.84\\
w/o corr. branch & \underline{1.31e-1} & \underline{1.35e-1} & \underline{49.12}\\
One-stage & 1.47e-1 & 1.52e-1 & 35.09\\
Full verifier & \textbf{1.26e-1} & \textbf{1.31e-1} & \textbf{52.63}\\
\bottomrule
\end{tabular}
\caption{\raggedright Verifier ablations on the fixed synthetic validation set.}
\label{tab:representation_ablation}

\end{minipage}
\end{table}
\endgroup

\Needspace{4\baselineskip}
\noindent\textbf{Research Question 4: Which verifier inputs and training stages improve prediction?}\par\nopagebreak[4]
\noindent We compare the full verifier with four variants: (i) \emph{without tree address}, which zeroes the formula encoder's depth/address contribution and retrains both stages; (ii) \emph{formula only}, which removes both numerical encoders while preserving the fusion-head interface; (iii) \emph{without correlation branch}, which removes the signed-correlation encoder; and (iv) \emph{one-stage}, which retains the full architecture and inputs but optimizes the Stage~2 objective from random initialization. The formula-only and no-correlation variants start from the full model's Stage~1 checkpoint.

\noindent Table~\ref{tab:representation_ablation} shows that every ablation worsens both the objective and MAE$_{\leq1}$ relative to the full verifier. Recall also decreases without numerical inputs, the correlation branch, or Stage~1, while the no-tree-address variant ties the full model. Together, these results indicate that each tested component is needed to attain the best overall verifier performance. Training details appear in Appendix~\ref{app:representation_ablation}.

\begingroup
\setlength{\intextsep}{0pt}
\setlength{\abovecaptionskip}{2pt}
\subsection{Case study: Physics 11}
In round 57 of Physics 11, a T-round selects a parent with worse validation NMSE (1.019 versus 0.0476) but lower predicted TED (0.504 versus 0.758). Its child separates sinusoidal damping and adds a time--position term, reaching test NMSE $2.09\times10^{-15}$. Figure~\ref{fig:po11_case_study} shows this choice. This example shows that the TED verifier can guide the LLM toward the correct formula.

\begin{figure}[H]
    \centering
    \includegraphics[width=0.95\linewidth]{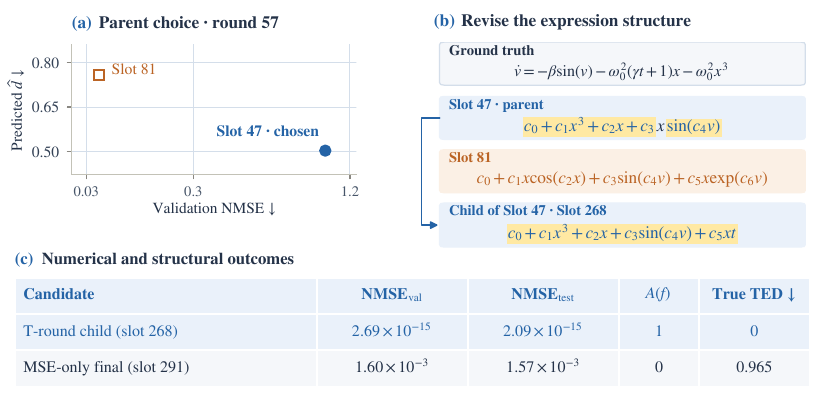}
    \caption{PO11 case study. (a) Structural feedback selects slot 47 despite its higher validation NMSE. (b) Ground truth, candidate structures, and the child generated from slot 47; $c_i$ denote fitted constants. (c) Numerical and structural outcomes.}
    \label{fig:po11_case_study}
\end{figure}
\endgroup
\pagebreak[0] 

\endgroup

\section{Conclusion and future work}
\label{sec:conclusion}
We introduced \ours{} through three contributions: a modified TED for structural supervision, a verifier that estimates structural distance to an unseen target from observations and candidate formulas, and an LLM-guided search agent that combines structural and numerical feedback to select and refine formulas. Limited computational resources led us to evaluate lightweight, low-cost LLMs. Future work will optimize the framework with the latest LLMs. We will also explore alternative training strategies, including post-training LLMs as structural verifiers or directly post-training the analyst to improve formula revision.
\par\newpage 

\section*{AI use statement}
AI was used for writing code, debugging, running some experiments, polishing the manuscript, checking mathematical derivations, literature searches, and identifying relevant studies for parts of the related work.

\section*{Reproducibility statement}
Appendices~\ref{app:distance}--\ref{app:search} specify the structural target, verifier inputs, optimization settings, and search decisions. Appendix~\ref{app:results} records measurements, denominators, and the reporting protocol. The accompanying writing artifacts contain the source-linked result summaries and scripts used to construct the tables and figures.

\section*{Ethics statement}
The study concerns equation discovery from benchmark observations and synthetic mathematical expressions. A plausible expression or a small structural score is not a certificate of a physical law. The formulas discovered by this model do not necessarily represent actual physical laws, and users should assess their correctness and validity.

\bibliography{iclr2027_conference}
\bibliographystyle{iclr2027_conference}
\clearpage
\appendix
\raggedbottom
\renewcommand{\floatpagefraction}{0.9}
\makeatletter
\setlength{\@fptop}{0pt}
\setlength{\@fpsep}{12pt}
\makeatother
\section{Structural target: implementation details}
\label{app:distance}
\subsection{Tree construction and preprocessing}
Variables retain their identities. Numeric-only subexpressions become a single \texttt{number} leaf; symbolic constants recognized by the parser become \texttt{param}. Function names and operation types supply internal labels. Nested Add and Mul nodes are flattened recursively. An AST is the representation; TED is the subsequent comparison algorithm.

The three preprocessing views in Equation~\ref{eq:target} are:
\begin{enumerate}
\item \textbf{Structure-preserving candidate view.} Parse the candidate without automatic evaluation, absorb parameter-only subexpressions, and collapse pure constants while preserving parse structure where possible. Process the reference using the ordinary evaluated view below.
\item \textbf{Ordinary evaluated view.} Use normal symbolic evaluation, absorb constant-only symbolic combinations, and collapse pure numeric subexpressions on both sides.
\item \textbf{Expanded view.} Absorb constants, expand products, expand power bases, combine powers, denest powers, collapse pure constants, and absorb remaining symbolic constant combinations.
\end{enumerate}
The expanded view applies forced power transformations, and TED is evaluated after this preprocessing. Each view is normalized by its own reference-tree size before taking the minimum. The result is not clipped to $[0,1]$.

\subsection{Costs and local assignment}
Node insertion and deletion cost one. Relabeling identical labels costs zero, as does exchanging \texttt{number} and \texttt{param}; other relabelings cost one. Whole-subtree insertion and deletion use Equation~\ref{eq:subtree}: for $s\leq5$ the cost is $s$, and larger subtrees receive a discount. This extended ordered computation is evaluated for every root pair.

When both roots belong to $\{\mathrm{Add},\mathrm{Mul}\}$, an additional child-assignment route is evaluated, including the root relabeling cost when necessary. Real-child matches recurse through the hybrid routine; padded matches pay a subtree insertion or deletion cost. The smaller ordered/assignment cost is retained. At other roots, the ordered routine is returned.

Changing a numeric exponent from two to three leaves the abstract power tree unchanged; changing $\sin$ to $\cos$ changes an operator label.

\subsection{Worked examples across the three views}
\label{app:ted_examples}
We illustrate five distinct mechanisms using a candidate $f$ and a reference $g$. Write $D_k$ for the raw hybrid edit cost and $n_k=|T(P_k^r(g))|$ for the reference-tree size in view $v_k$. The normalized score is $D_k/n_k$, and the final TED is the minimum across the three views. Table~\ref{tab:ted_worked_examples} reports all three calculations. Ties are resolved in the order $v_0,v_1,v_2$. Figures~\ref{fig:ted_worked_examples_ac} and~\ref{fig:ted_worked_examples_de} draw the actual candidate and reference trees after the selected preprocessing; $N$ denotes a \texttt{number} leaf.

\begin{table}[!htbp]
\centering
\caption{Worked TED calculations. Each view reports the raw edit cost divided by its reference-tree size. The letters match the tree diagrams.}
\label{tab:ted_worked_examples}
\small
\setlength{\tabcolsep}{4pt}
\begin{tabular}{cllcccc}
\toprule
 & Candidate $f$ & Reference $g$ & $v_0$ & $v_1$ & $v_2$ & TED \\
\midrule
(a) & $x+(y+z)$ & $(x+y)+z$ & $0/4$ & $0/4$ & $0/4$ & $0$ \\
(b) & $\sin(x+y)$ & $\cos(x+y)$ & $1/4$ & $1/4$ & $1/4$ & $1/4$ \\
(c) & $x(y+z)$ & $xy+xz$ & $5/7$ & $5/7$ & $0/7$ & $0$ \\
(d) & $(x+y)^2$ & $x^2+2xy+y^2$ & $9/11$ & $9/11$ & $0/11$ & $0$ \\
(e) & $x+\sin(yz+w)$ & $x$ & $6.6/1$ & $6.6/1$ & $6.6/1$ & $6.6$ \\
\bottomrule
\end{tabular}
\end{table}

\noindent\textbf{(a) Associativity.}
For $f=x+(y+z)$ and $g=(x+y)+z$, flattening nested Add nodes produces one Add root with leaves $x,y,z$ on both sides. Each reference tree has four nodes, and no edits are needed: $v_0=v_1=v_2=0/4=0$. The selected view is $v_0$.

\noindent\textbf{(b) Function relabeling.}
For $f=\sin(x+y)$ and $g=\cos(x+y)$, the Add subtree and its variable leaves agree. Only the root function label changes, costing one. All three reference trees have four nodes, so $v_0=v_1=v_2=1/4$; $v_0$ is selected. The highlighted roots show this single edit.

\noindent\textbf{(c) Product expansion.}
For $f=x(y+z)$ and $g=xy+xz$, views $v_0$ and $v_1$ retain a Mul root on the candidate and an Add root on the reference. A minimum-cost assignment changes the root at cost one, matches $x$ to $xy$ at cost two (insert Mul and $y$), and matches $y+z$ to $xz$ at cost two (Add to Mul and $y$ to $x$). Thus $D_0=D_1=1+2+2=5$ and $n_0=n_1=7$. View $v_2$ distributes the product, making both trees identical with seven nodes: $v_2=0/7$. The selected score is zero.

\begin{figure}[!htbp]
\centering
\includegraphics[width=\linewidth]{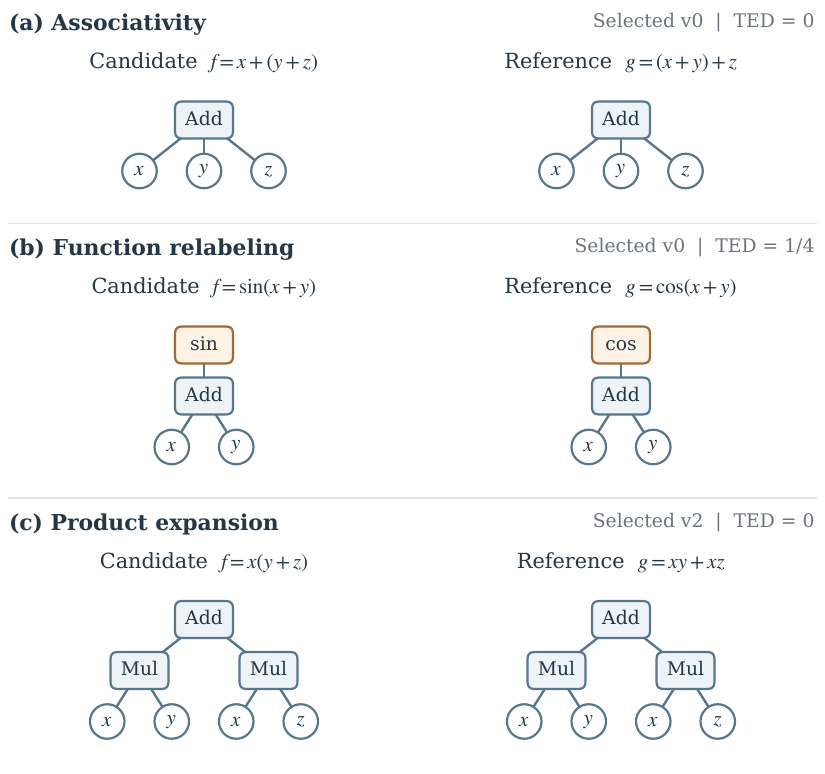}
\caption{Selected abstract trees for examples (a)--(c). Input formulas appear above each tree; the diagrams show the trees after the indicated preprocessing. Flattening removes grouping differences in (a), the highlighted function labels require one edit in (b), and product expansion gives identical trees in (c).}
\label{fig:ted_worked_examples_ac}
\end{figure}

\noindent\textbf{(d) Polynomial expansion.}
For $f=(x+y)^2$ and $g=x^2+2xy+y^2$, views $v_0$ and $v_1$ compare $\mathrm{Pow}(\mathrm{Add}(x,y),N)$ with an Add root whose children are $\mathrm{Pow}(x,N)$, $\mathrm{Pow}(y,N)$, and $\mathrm{Mul}(N,x,y)$. An optimal ordered edit changes the root at cost one; converts $\mathrm{Add}(x,y)$ to $\mathrm{Pow}(x,N)$ at cost two; converts the remaining $N$ to $\mathrm{Pow}(y,N)$ at cost two; and inserts the four-node mixed term at cost four. Hence $D_0=D_1=1+2+2+4=9$, with $n_0=n_1=11$. View $v_2$ expands the square, giving identical 11-node trees and the selected score $0/11=0$. Numeric exponent and coefficient leaves are all labeled $N$.

\noindent\textbf{(e) Subtree deletion discount.}
For $f=x+\sin(yz+w)$ and $g=x$, all views retain the extra branch. The outlined branch has six nodes: sin, Add, Mul, $y,z,w$. Deleting it costs $\kappa(6)=0.6(6-5)+5=5.6$; deleting the outer Add root costs one more and leaves $x$. Therefore $D_0=D_1=D_2=6.6$, while every reference tree contains only one node. All three normalized scores are $6.6/1=6.6$, and $v_0$ is selected. This example also shows why normalization by the reference size does not bound the score by one.

\begin{figure}[!htbp]
\centering
\includegraphics[width=\linewidth]{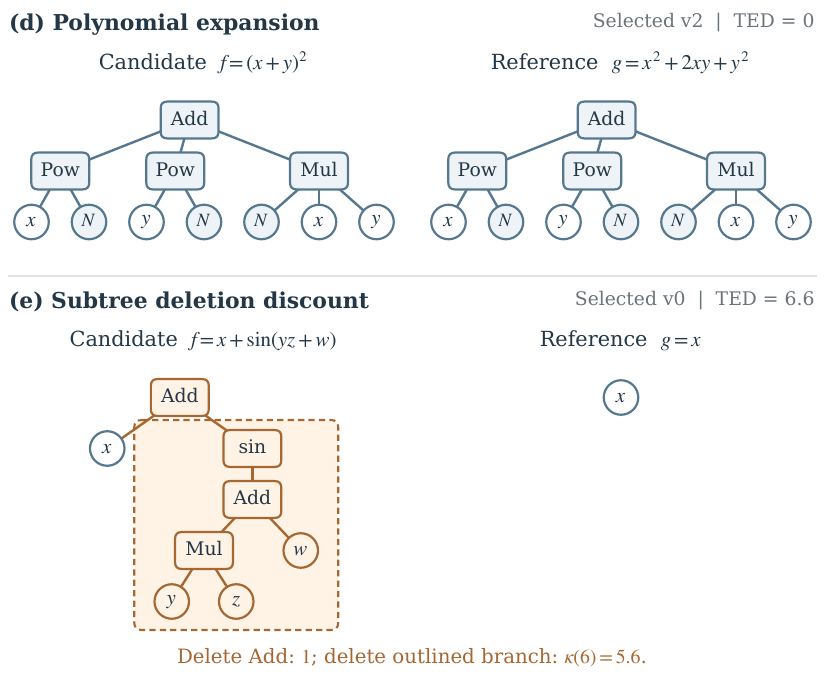}
\caption{Selected abstract trees for examples (d)--(e). In (d), expanding the square yields identical trees; $N$ denotes an abstract numeric leaf. In (e), the dashed outline contains the six-node branch deleted at discounted cost 5.6; the highlighted Add root incurs another unit deletion cost. Only the variable $x$ is retained.}
\label{fig:ted_worked_examples_de}
\end{figure}

\begingroup
\raggedbottom
\subsection{Additional examples of partial structural matches}
\label{app:ted_additional_examples}
Table~\ref{tab:ted_additional_examples} gives six further formula pairs, continuing the notation above. All select $v_0$. The tree diagrams preserve the child order returned by preprocessing; commutative assignment can still match children across different positions. Orange nodes indicate deletions, insertions, or relabelings along one minimum-cost edit path. As before, $N$ denotes a \texttt{number} leaf.

\begin{table}[!htbp]
\centering
\caption{Additional worked TED calculations. Each view reports the raw edit cost divided by its reference-tree size. All six pairs select $v_0$, using the first view in a tie.}
\label{tab:ted_additional_examples}
\small\setlength{\tabcolsep}{4pt}
\begin{tabular}{cllcccc}
\toprule
 & Candidate $f$ & Reference $g$ & $v_0$ & $v_1$ & $v_2$ & TED \\
\midrule
(f) & $x^2\sin y$ & $\sin y$ & $4/2$ & $4/2$ & $4/2$ & $2$ \\
(g) & $(x+y)(x+z)$ & $xy+yz$ & $4/7$ & $4/7$ & $6/7$ & $4/7$ \\
(h) & $(x+y)^3$ & $x^3+y^3$ & $5/7$ & $5/7$ & $11.2/7$ & $5/7$ \\
(i) & $\sqrt{x+y}$ & $\sqrt{x}$ & $2/3$ & $2/3$ & $2/3$ & $2/3$ \\
(j) & $(x+y)(x+z)$ & $x(x+y)$ & $2/5$ & $2/5$ & $6/7$ & $2/5$ \\
(k) & $(x+y)^2$ & $(x+y)^3$ & $0/5$ & $0/5$ & $7.6/19$ & $0$ \\
\bottomrule
\end{tabular}
\end{table}

\noindent\textbf{(f) Extra multiplicative factor.}
For $f=x^2\sin y$ and $g=\sin y$, delete the three-node factor $\mathrm{Pow}(x,N)$ and the outer Mul root, retaining the entire $\sin y$ subtree. The cost is $3+1=4$, and the reference has two nodes. All views therefore give $4/2=2$.

\noindent\textbf{(g) Product versus sum.}
For $f=(x+y)(x+z)$ and $g=xy+yz$, the unexpanded view changes the Mul root to Add at cost one, changes $\mathrm{Add}(x,y)$ to $\mathrm{Mul}(x,y)$ at cost one, and changes $\mathrm{Add}(x,z)$ to $\mathrm{Mul}(y,z)$ at cost two. Thus $v_0=v_1=(1+1+2)/7=4/7$. Expansion gives $x^2+xy+xz+yz$; deleting the extra three-node terms $x^2$ and $xz$ gives $v_2=6/7$. The unexpanded view has the smaller distance.

\begin{figure}[!htbp]
\centering
\includegraphics[width=\linewidth]{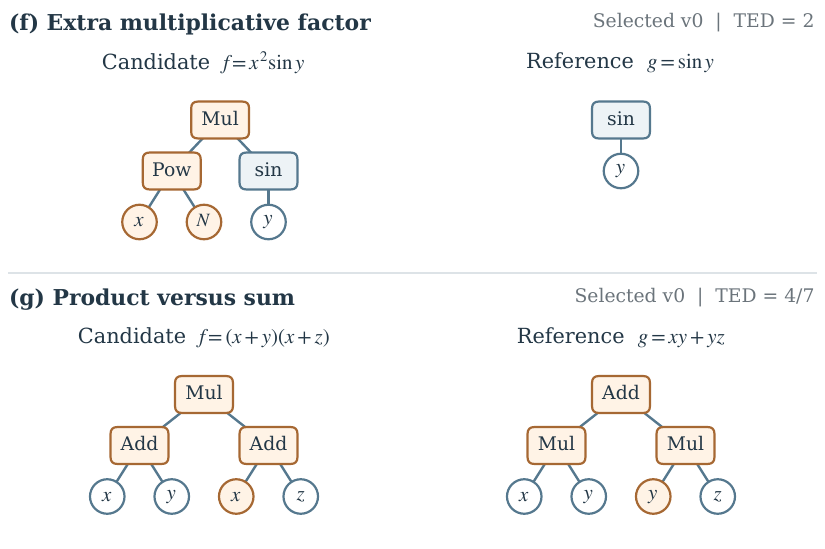}
\caption{Selected $v_0$ trees for (f)--(g). Orange marks nodes edited along a minimum-cost path. In (f), the extra power subtree and Mul root are deleted. In (g), four relabelings transform the product of sums into the reference sum of products.}
\label{fig:ted_additional_examples_fg}
\end{figure}

\noindent\textbf{(h) Cube versus separate cubes.}
For $f=(x+y)^3$ and $g=x^3+y^3$, $v_0$ and $v_1$ compare $\mathrm{Pow}(\mathrm{Add}(x,y),N)$ with $\mathrm{Add}(\mathrm{Pow}(x,N),\mathrm{Pow}(y,N))$. Change the root at cost one, convert $\mathrm{Add}(x,y)$ to $\mathrm{Pow}(x,N)$ at cost two, and turn the remaining $N$ into $\mathrm{Pow}(y,N)$ at cost two. Hence $v_0=v_1=5/7$. After expansion, the two additional mixed terms $3x^2y$ and $3xy^2$ each have six nodes. Their discounted deletion costs give $v_2=2\kappa(6)/7=11.2/7=1.6$, so the final score is $5/7$.

\noindent\textbf{(i) Square-root argument.}
For $f=\sqrt{x+y}$ and $g=\sqrt{x}$, the trees are $\mathrm{Pow}(\mathrm{Add}(x,y),N)$ and $\mathrm{Pow}(x,N)$. The exponent $1/2$ becomes one numeric leaf. Retain the Pow root, $x$, and $N$; delete the inner Add and $y$. All views give $2/3$.

\begin{figure}[!htbp]
\centering
\includegraphics[width=\linewidth]{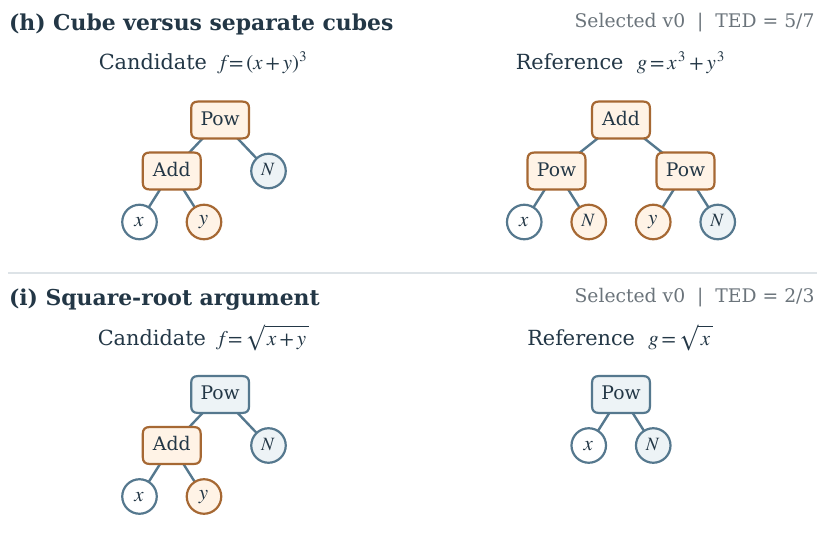}
\caption{Selected $v_0$ trees for (h)--(i). In (h), orange marks three relabelings and two inserted nodes; the original exponent leaf is retained as the reference's rightmost $N$. In (i), deleting the orange Add and $y$ nodes leaves the reference square-root tree. Numeric exponents are represented by $N$.}
\label{fig:ted_additional_examples_hi}
\end{figure}

\noindent\textbf{(j) Shared additive factor.}
For $f=(x+y)(x+z)$ and $g=x(x+y)$, commutative assignment matches the common $\mathrm{Add}(x,y)$ subtree across child positions. In the other factor, delete Add and $z$ from $x+z$, retaining $x$. The raw cost is two and the reference has five nodes, giving $v_0=v_1=2/5$. Expansion produces $x^2+xy+xz+yz$ and $x^2+xy$; deleting the two extra three-node products gives $v_2=6/7$. The final score is $2/5$.

\noindent\textbf{(k) Numeric exponent abstraction.}
For $f=(x+y)^2$ and $g=(x+y)^3$, $v_0$ and $v_1$ both produce $\mathrm{Pow}(\mathrm{Add}(x,y),N)$, so their scores are $0/5$. In $v_2$, the pure-power terms still match because numeric exponents are abstracted. Matching the candidate mixed term $2xy$ to one reference mixed term requires two insertions; inserting the remaining six-node mixed term costs $\kappa(6)=5.6$. The reference has 19 nodes, giving $v_2=(2+5.6)/19=0.4$. The final TED is zero, illustrating that zero structural distance does not imply algebraic equivalence.

\begin{figure}[!htbp]
\centering
\includegraphics[width=\linewidth]{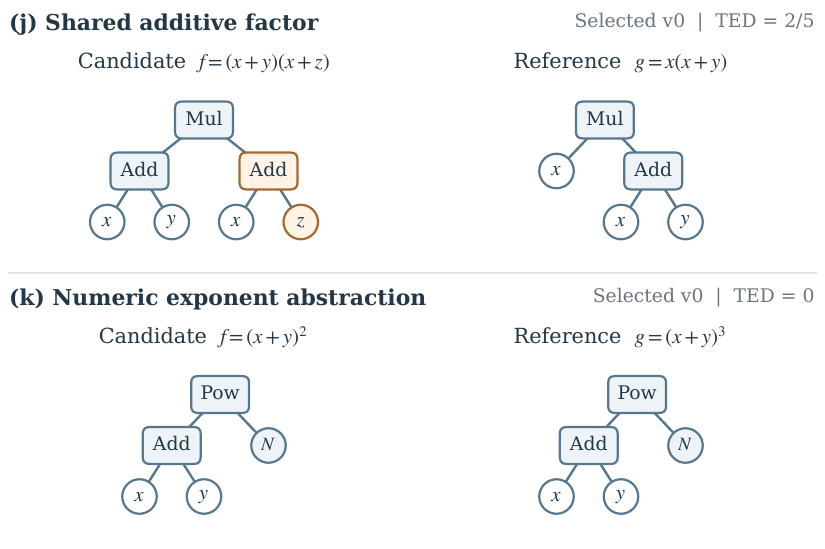}
\caption{Selected $v_0$ trees for (j)--(k). In (j), commutative matching preserves the common additive factor; deleting the orange Add and $z$ nodes supplies the remaining $x$ factor. In (k), both exponent values become the same numeric leaf $N$, making the abstract trees identical.}
\label{fig:ted_additional_examples_jk}
\end{figure}
\par\endgroup

\section{Synthetic formula and observation generation}
\label{app:generator}
We describe the archived August 30 configuration. Listed probabilities are rounded proposals; feasibility checks and training resampling alter the resulting distributions. Categorical weights are renormalized over permitted outcomes.

\subsection{Empirical priors and formula size}
\label{app:generator_priors}
\begin{enumerate}
\item \textbf{Source formulas and weighting.}
Priors use 32,177 LLM-SR formulas from 153 problems, 20,400 PySR formulas from 239 problems, and 120 folded SRBench/Feynman targets, each treated as one problem. Variable, root, operator, and constant statistics give equal weight to sources, problems within each source, and formulas within each problem:
\begin{equation}
 w(g,p,s)=\frac{1}{3|\mathcal P_s||\mathcal F_{s,p}|},
\end{equation}
Here $\mathcal P_s$ contains source $s$'s problems, and $\mathcal F_{s,p}$ contains its formulas for problem $p$. Before counting, negative powers use inverses, and exponents $1/2$, $2$, and $3$ use square-root, square, and cube operators.

\item \textbf{Node-count and operator-count prior.}
Let $N$ count all nodes, $B$ binary operators (e.g., addition or division), and $U$ unary operators (e.g., sine or logarithm). With $B+1$ leaves,
\[
 N=2B+U+1.
\]
For this joint size prior only, the source weights are
$(\pi_{\mathrm{GT}},\pi_{\mathrm{PySR}},\pi_{\mathrm{LLM\text{-}SR}})=(0.2,0.2,0.6)$:
\begin{equation}
 P_0(N=n,B=b)=
 \sum_{s,p,g}\frac{\pi_s}{|\mathcal P_s||\mathcal F_{s,p}|}
 \mathbf 1[N(g)=n,\ B(g)=b].
 \label{eq:generator_joint}
\end{equation}
Exclude formulas exceeding 100 nodes, then adjust the length marginal while preserving $P_0(B\mid N)$. Draw integer $N$ from the adjusted row sums and $B$ from that row's feasible entries. Table~\ref{tab:generator_priors} groups lengths for display only. The legacy length-bin floor of $0.001$ is inactive.

\item \textbf{Variable counts, root families, and feasibility.}
Limit formulas to 100 nodes and 10 variables. Clip variable-bin upper endpoints to $\min(10,\lfloor(N+1)/2\rfloor)$, remove empty bins, then select a bin and draw an integer uniformly within it. A one-node tree is a variable; $N=2$ forces a unary root; $N>2$ excludes a leaf root. The $B+1$ leaves must accommodate all variables after reserving exponent leaves. A unary root or required inverse needs a unary operator; a non-division binary root with required division needs $B\geq2$. Retry infeasible proposals.
\end{enumerate}

\subsubsection{Theory of mixture priors and conditional sampling}
\label{app:generator_prior_theory}

Let $P_s$ be the empirical law obtained by averaging formulas within each problem and then problems within source $s$, as specified above. The mixture is $P_\pi=\sum_s\pi_sP_s$. For the joint size prior, $\pi=(0.2,0.2,0.6)$ in GT/PySR/LLM-SR order, and each source contains only formulas retained after the 100-node cutoff and their nonempty problems. Operator statistics instead use $\pi_s=1/3$.

\begin{proposition}[Conditional sampling from the mixed size prior]
\label{prop:generator_bayes_mixture}
For any sampled length $n$ with $P_\pi(N=n)>0$, the source weights conditional on that length and the resulting binary-count law are
\begin{equation}
 \begin{aligned}
 \alpha_s(n)&=\frac{\pi_sP_s(N=n)}{\sum_t\pi_tP_t(N=n)},\\
 P_0(B=b\mid N=n)
 &=\frac{\sum_s\pi_sP_s(N=n,B=b)}{\sum_s\pi_sP_s(N=n)}
 =\sum_s\alpha_s(n)P_s(B=b\mid N=n).
 \end{aligned}
 \label{eq:generator_bayes_mixture}
\end{equation}
Sources with no mass at $n$ are omitted. Thus normalizing the mixed joint row implements a source draw from $\alpha(n)$ followed by a source-specific conditional count draw.
\end{proposition}
\begin{proof}
Let $S$ be the selected source and $Z_n=\sum_s\pi_sP_s(N=n)>0$. The joint and marginal masses are
\[
 \begin{aligned}
 P_\pi(S=s,N=n,B=b)&=\pi_sP_s(N=n,B=b),\\
 P_\pi(S=s,N=n)&=\sum_b\pi_sP_s(N=n,B=b)=\pi_sP_s(N=n),\\
 P_\pi(N=n)&=\sum_sP_\pi(S=s,N=n)=Z_n.
 \end{aligned}
\]
Hence
\[
 \alpha_s(n)=\frac{P_\pi(S=s,N=n)}{P_\pi(N=n)}
 =\frac{\pi_sP_s(N=n)}{Z_n}\geq0,
 \qquad
 \sum_s\alpha_s(n)=\frac{Z_n}{Z_n}=1.
\]
A zero-mass source satisfies $P_s(N=n)=0\Rightarrow P_s(N=n,B=b)=0$ for every $b$. On the remaining sources,
\[
 \begin{aligned}
 P_0(B=b\mid N=n)
 &=\frac{\sum_s\pi_sP_s(N=n,B=b)}{Z_n}\\
 &=\sum_{s:P_s(N=n)>0}
   \frac{\pi_sP_s(N=n)}{Z_n}P_s(B=b\mid N=n)\\
 &=\sum_{s:P_s(N=n)>0}\alpha_s(n)P_s(B=b\mid N=n).
 \end{aligned}
\]
The resulting law is normalized and equals the two-stage sampling law:
\[
 \begin{aligned}
 \sum_bP_0(B=b\mid N=n)
 &=\sum_{s:P_s(N=n)>0}\alpha_s(n)
   \underbrace{\sum_bP_s(B=b\mid N=n)}_{1}=1,\\
 \Pr_{\mathrm{two\text{-}stage}}(B=b\mid n)
 &=\sum_{s:P_s(N=n)>0}
   \Pr(S=s\mid n)\Pr(B=b\mid S=s,n)\\
 &=\sum_{s:P_s(N=n)>0}\alpha_s(n)P_s(B=b\mid N=n)
 =P_0(B=b\mid N=n).
 \end{aligned}
\]
\end{proof}
This justifies mixing joint counts before normalization: conditioning on length changes the appropriate source weights from $\pi_s$ to $\alpha_s(n)$.

\begin{proposition}[Minimal-change length reweighting and feasible sampling]
\label{prop:generator_marginal_projection}
Let $\bar P_0$ be the normalized retained joint size law, $p_n=\bar P_0(N=n)$, and $q$ the adjusted length distribution, supported where $p_n>0$. Row scaling gives
\begin{equation}
 Q(n,b)=\frac{q_n}{p_n}\bar P_0(n,b)
       =q_n\bar P_0(b\mid n),
 \label{eq:generator_length_reweight}
\end{equation}
with zero mass on unused rows. This law has length marginal $q$ and preserves the binary-count conditional at each used length. Among joint laws $R(N,B)$ supported on $\bar P_0$ with length marginal $q$, it uniquely minimizes $D_{\mathrm{KL}}(R\Vert\bar P_0)$, attaining the minimum $D_{\mathrm{KL}}(q\Vert p)$.

Let $h$ collect the sampled variable count, root family, and division/inverse requirements, and let $\mathcal B(n,h)$ be the admitted binary counts. For $q_n>0$ and positive admissible row mass, restricting either $Q(n,\cdot)$ or $\bar P_0(n,\cdot)$ gives the same sampling kernel:
\begin{equation}
 K_B(b\mid n,h)
 =\frac{\bar P_0(n,b)\,\mathbf 1[b\in\mathcal B(n,h)]}
 {\sum_{b'\in\mathcal B(n,h)}\bar P_0(n,b')}.
 \label{eq:generator_feasible_conditional}
\end{equation}
\end{proposition}
\begin{proof}
\textbf{Marginal and conditional laws.}
For $q_n>0$, the support assumption gives $p_n>0$, and
\[
 \begin{aligned}
 \sum_bQ(n,b)&=\frac{q_n}{p_n}\sum_b\bar P_0(n,b)
              =\frac{q_n}{p_n}p_n=q_n,\\
 \sum_{n,b}Q(n,b)&=\sum_nq_n=1,\\
 Q(b\mid n)&=\frac{Q(n,b)}{q_n}
             =\frac{\bar P_0(n,b)}{p_n}=\bar P_0(b\mid n),\\
 \Pr_{\mathrm{sampler}}(N=n,B=b)
 &=q_n\bar P_0(b\mid n)=Q(n,b).
 \end{aligned}
\]
Rows with $q_n=0$ have $Q(n,b)=0$ for all $b$.

\textbf{KL decomposition.}
For an admissible $R$, set $r_n(b)=R(b\mid n)$ and $t_n(b)=\bar P_0(b\mid n)$ on rows with $q_n>0$. Then
\[
 R(n,b)=q_nr_n(b),\qquad
 \bar P_0(n,b)=p_nt_n(b),\qquad
 \sum_br_n(b)=\sum_bt_n(b)=1.
\]
In KL sums, zero-probability terms are interpreted as zero, using $0\log0=0$. Since $R$ is supported on $\bar P_0$, every denominator at a positive-mass term is positive. Expanding the logarithm gives
\begin{equation}
 \begin{aligned}
 D_{\mathrm{KL}}(R\Vert\bar P_0)
 &=\sum_{n:q_n>0}\sum_bq_nr_n(b)
   \log\frac{q_nr_n(b)}{p_nt_n(b)}\\
 &=\sum_{n:q_n>0}q_n\sum_br_n(b)
   \left(\log\frac{q_n}{p_n}+\log\frac{r_n(b)}{t_n(b)}\right)\\
 &=\sum_{n:q_n>0}q_n\log\frac{q_n}{p_n}
   \underbrace{\sum_br_n(b)}_{1}
   +\sum_{n:q_n>0}q_n\sum_br_n(b)\log\frac{r_n(b)}{t_n(b)}\\
 &=D_{\mathrm{KL}}(q\Vert p)
   +\sum_{n:q_n>0}q_nD_{\mathrm{KL}}(r_n\Vert t_n).
 \end{aligned}
 \label{eq:generator_kl_decomposition}
\end{equation}

\textbf{Optimality and uniqueness.}
The standard inequality $x\log x-x+1\geq0$ for $x\geq0$, with equality only at $x=1$, yields
\[
 \begin{aligned}
 D_{\mathrm{KL}}(r_n\Vert t_n)
 &=\sum_{b:t_n(b)>0}t_n(b)
   \left[
    \frac{r_n(b)}{t_n(b)}\log\frac{r_n(b)}{t_n(b)}
    -\frac{r_n(b)}{t_n(b)}+1
   \right]\geq0,\\
 D_{\mathrm{KL}}(r_n\Vert t_n)=0
 &\quad\Longleftrightarrow\quad r_n=t_n.
 \end{aligned}
\]
The inserted linear terms sum to $-\sum_br_n(b)+\sum_bt_n(b)=0$. Therefore
\[
 \begin{aligned}
 D_{\mathrm{KL}}(R\Vert\bar P_0)
 &\geq D_{\mathrm{KL}}(q\Vert p)
   =D_{\mathrm{KL}}(Q\Vert\bar P_0),\\
 D_{\mathrm{KL}}(R\Vert\bar P_0)=D_{\mathrm{KL}}(q\Vert p)
 &\Longleftrightarrow
   r_n=t_n\ \text{for every }n\text{ with }q_n>0\\
 &\Longleftrightarrow
   R(n,b)=q_n\bar P_0(b\mid n)=Q(n,b)\ \text{on all used rows}\\
 &\Longleftrightarrow R=Q.
 \end{aligned}
\]
The last equivalence also uses $q_n=0\Rightarrow R(n,b)=Q(n,b)=0$ for all $b$.

\textbf{Feasibility conditioning.}
At fixed $(n,h)$ with $q_n>0$, write $\mathcal B=\mathcal B(n,h)$ and $A_{n,h}=\sum_{b'\in\mathcal B}\bar P_0(n,b')>0$. Then
\[
 \begin{aligned}
 K_B(b\mid n,h)
 &=\frac{Q(b\mid n)\mathbf 1[b\in\mathcal B]}
         {\sum_{b'\in\mathcal B}Q(b'\mid n)}\\
 &=\frac{Q(n,b)\mathbf 1[b\in\mathcal B]}
         {\sum_{b'\in\mathcal B}Q(n,b')}\\
 &=\frac{(q_n/p_n)\bar P_0(n,b)\mathbf 1[b\in\mathcal B]}
         {(q_n/p_n)A_{n,h}}
 =\frac{\bar P_0(n,b)\mathbf 1[b\in\mathcal B]}{A_{n,h}},\\
 \sum_bK_B(b\mid n,h)&=\frac{A_{n,h}}{A_{n,h}}=1.
 \end{aligned}
\]
Thus feasibility conditioning cancels both the length marginal and the row multiplier, giving Equation~\ref{eq:generator_feasible_conditional}.
\end{proof}
The KL characterization makes row scaling a minimal-change update of the empirical joint law under the prescribed length marginal; the sampler realizes this minimizer directly. A zero denominator in the feasible kernel triggers a retry. This result concerns the size proposal and its local feasibility mask: subsequent operator assignment, exponent-leaf reservation, and formula filtering add further constraints. Structural retries may redraw $N$, and simplification may change its value, so the final accepted length distribution need not equal $q$. The separate empirical priors and sequential masks do not define a single empirically learned joint law over all formula properties.

\Needspace{5\baselineskip}
\begin{proposition}[Conditional operator priors from weighted occurrences]
\label{prop:generator_operator_bayes}
For operator $o$ of arity $a$, let $m_s(o)=\mathbb E_s[c_o(G)]$ and $m_s(a)=\mathbb E_s[c_a(G)]$ under the same problem/formula hierarchy. If $\sum_s\pi_sm_s(a)>0$, mixing occurrence counts and conditioning on arity gives
\begin{equation}
 P(o\mid a)=\frac{\sum_s\pi_sm_s(o)}{\sum_s\pi_sm_s(a)}
 =\sum_{s:m_s(a)>0}
 \underbrace{\frac{\pi_sm_s(a)}{\sum_t\pi_tm_t(a)}}_{\alpha_s(a)}
 \frac{m_s(o)}{m_s(a)}.
 \label{eq:generator_operator_bayes}
\end{equation}
Here $\alpha_s(a)$ is the source probability conditional on an arity-$a$ occurrence.
\end{proposition}
\begin{proof}
Let $\mathcal O_a$ be the operators of arity $a$ and $\mu(s,o)=\pi_sm_s(o)$ their weighted occurrence mass. Since $c_a(G)=\sum_{o\in\mathcal O_a}c_o(G)$,
\[
 m_s(a)=\sum_{o\in\mathcal O_a}m_s(o),\qquad
 M_a:=\sum_s\sum_{o\in\mathcal O_a}\mu(s,o)
     =\sum_s\pi_sm_s(a)>0.
\]
For the occurrence source $S$, label $O$, and arity $A$,
\[
 \begin{aligned}
 P(S=s,O=o\mid A=a)&=\frac{\mu(s,o)}{M_a}
                    =\frac{\pi_sm_s(o)}{M_a},\qquad o\in\mathcal O_a,\\
 \sum_s\sum_{o\in\mathcal O_a}P(S=s,O=o\mid A=a)&=\frac{M_a}{M_a}=1,\\
 P(S=s\mid A=a)
 &=\sum_{o\in\mathcal O_a}\frac{\pi_sm_s(o)}{M_a}
   =\frac{\pi_sm_s(a)}{M_a}=\alpha_s(a).
 \end{aligned}
\]
For $\pi_sm_s(a)>0$, conditioning within the source gives
\[
 P(O=o\mid S=s,A=a)
 =\frac{P(S=s,O=o\mid A=a)}{P(S=s\mid A=a)}
 =\frac{\pi_sm_s(o)/M_a}{\pi_sm_s(a)/M_a}
 =\frac{m_s(o)}{m_s(a)}.
\]
Consequently,
\[
 \begin{aligned}
 P(O=o\mid A=a)
 &=\sum_{s:m_s(a)>0}\alpha_s(a)\frac{m_s(o)}{m_s(a)}\\
 &=\sum_{s:m_s(a)>0}
   \frac{\pi_sm_s(a)}{M_a}\frac{m_s(o)}{m_s(a)}
 =\frac{\sum_s\pi_sm_s(o)}{M_a},\\
 \sum_{o\in\mathcal O_a}P(O=o\mid A=a)
 &=\frac{\sum_s\pi_s\sum_{o\in\mathcal O_a}m_s(o)}{M_a}
   =\frac{M_a}{M_a}=1.
 \end{aligned}
\]
Here $m_s(a)=0\Rightarrow m_s(o)=0$ for every $o\in\mathcal O_a$; sources of zero weighted mass contribute zero and require no conditional draw.
\end{proof}
Even with equal source weights $\pi_s=1/3$, the conditional source contribution depends on its weighted arity count. Thus the implementation mixes occurrence counts before dividing, rather than averaging per-formula operator ratios. The unary floor and contextual exclusions are subsequent modifications of this base law.

\begin{table}[!htbp]
\caption{Effective size, variable-count, and root-family proposals before feasibility constraints. The variable and root columns are separate categorical distributions.}
\label{tab:generator_priors}
\centering\small
\begin{tabular}{llr@{\qquad}llr}
\toprule
Quantity & Outcome & Probability & Quantity & Outcome & Probability\\
\midrule
$N$ & 1--2 & 0.001000000 & Variables & 1 & 0.127767302\\
 & 3--7 & 0.030000000 & & 2 & 0.352240995\\
 & 8--15 & 0.321712885 & & 3 & 0.282177813\\
 & 16--30 & 0.440506600 & & 4--5 & 0.196282613\\
 & 31--50 & 0.184595015 & & 6--8 & 0.041531276\\
 & 51--70 & 0.018782916 & & 9--10 & 0\\
 & 71--85 & 0.002822544 & Root & additive & 0.624594682\\
 & 86--100 & 0.000580040 & & multiplicative & 0.101696355\\
 & & & & rational & 0.233553641\\
 & & & & power & 0.000989314\\
 & & & & unary composition & 0.038926974\\
 & & & & leaf & 0.000239034\\
\bottomrule
\end{tabular}
\end{table}

\subsection{Tree shape and operator assignment}
\label{app:generator_operators}
\begin{enumerate}
\item \textbf{Ordered tree shape.}
For $B$ binary operators, there are $C_B=\binom{2B}{B}/(B+1)$ ordered skeletons. Recursively choose left size $i$ and right size $B-1-i$ with probability $C_iC_{B-1-i}/C_B$, giving each skeleton probability $1/C_B$.

\item \textbf{Placement of the remaining operators.}
Across $2B+1$ skeleton positions, distribute $U=N-2B-1$ unary operators using a uniform weak composition $(u_1,\ldots,u_{2B+1})$ summing to $U$; wrap position $j$ in $u_j$ operators. A unary root reserves one operator before distributing $U-1$; a binary root receives zero, with $U$ distributed elsewhere.

\item \textbf{Operator probabilities.}
For an operator $o$ of arity $a$, the empirical weight is
\[
 P(o\mid a)=
 \frac{\sum_{s,p,g}w(g,p,s)c_o(g)}
 {\sum_{s,p,g}w(g,p,s)c_a(g)},
\]
Here $c_o$ counts operator $o$ and $c_a$ counts all arity-$a$ operators. Floor rare unary probabilities at $0.001$ and proportionally rescale the rest (Table~\ref{tab:generator_operators}).

\item \textbf{Root choices and division/inverse requirements.}
An additive root chooses \texttt{Add}/\texttt{Sub} with probabilities
$(0.656837367,0.343162633)$; multiplicative, rational, and power roots force
\texttt{Mul}, \texttt{Div}, and \texttt{Pow}, respectively.
The formula-level division proposal is $0.592371744$ and the inverse proposal is $0.024717228$. A rational root already requires division; otherwise, the probability of requiring it is
\[
 \frac{0.592371744-0.233553641}{1-0.233553641}=0.468158142.
\]
For $N=1$, disable both requirements. A unary root is \texttt{Inv} if required; otherwise exclude \texttt{Inv}. Force any still-required \texttt{Div}/\texttt{Inv} at a uniformly chosen available position; exclude unrequired ones from unassigned nodes. Multiply the \texttt{Pow} weight by $0.15$ for nonleaf exponents. Renormalize remaining weights and apply the checks in Appendix~\ref{app:generator_filter}.
\end{enumerate}

\subsubsection{Theory of the Catalan shape sampler}
\label{app:generator_tree_theory}

Condition on the sampled counts $B=b$ and $U=u$, so that $N=2b+u+1$. We analyze ordered, unlabeled tree shapes; left and right child positions remain distinct even when their eventual operator is commutative.

We use the standard Catalan count $C_b=\binom{2b}{b}/(b+1)$ for ordered full binary skeletons, with $C_0=1$ and $C_b=\sum_{i=0}^{b-1}C_iC_{b-1-i}$ for $b\geq1$~\citep[pp.~6--7]{flajolet2009analytic}.

\begin{proposition}[Uniform sampling of ordered binary skeletons]
\label{prop:generator_catalan_uniform}
At each binary node with a remaining budget $b\geq1$, draw the left-subtree budget $I$ with
\begin{equation}
 P(I=i\mid b)=\frac{C_iC_{b-1-i}}{C_b},
 \qquad i=0,\ldots,b-1,
 \label{eq:generator_catalan_split}
\end{equation}
then draw both subtrees independently by the same rule. Every ordered full binary skeleton with $b$ binary nodes has probability $1/C_b$.
\end{proposition}
\begin{proof}
\textbf{Valid split probabilities.}
The stated Catalan recurrence gives
\[
 P(I=i\mid b)=\frac{C_iC_{b-1-i}}{C_b}\geq0,\qquad
 \sum_{i=0}^{b-1}P(I=i\mid b)
 =\frac{\sum_{i=0}^{b-1}C_iC_{b-1-i}}{C_b}=1.
\]
Each recursive call has fewer binary nodes:
\[
 0\leq I\leq b-1,\qquad 0\leq b-1-I\leq b-1,
\]
so recursion terminates at the unique $b=0$ leaf.

\textbf{Uniformity by induction.}
For $b=0$, $P(\mathrm{leaf}\mid0)=1=1/C_0$. Assume
\[
 P(T'\mid k)=\frac{1}{C_k}
 \qquad\text{for every }0\leq k<b
 \text{ and every skeleton }T'\text{ of size }k.
\]
A skeleton $T$ of size $b$ uniquely determines its ordered subtrees and split:
\[
 T\longleftrightarrow(i,T_L,T_R),\qquad
 |T_L|_{\mathrm{binary}}=i,\quad
 |T_R|_{\mathrm{binary}}=j=b-1-i.
\]
Conditional independence of the recursive draws and the induction hypothesis imply
\[
 \begin{aligned}
 P(T\mid b)
 &=P(I=i\mid b)P(T_L,T_R\mid I=i,b)\\
 &=P(I=i\mid b)P(T_L\mid i)P(T_R\mid j)\\
 &=\frac{C_iC_j}{C_b}\frac{1}{C_i}\frac{1}{C_j}
 =\frac{1}{C_b}.
 \end{aligned}
\]
There is no second construction path for the same ordered tree. For the set $\mathcal S_b$ of all such skeletons,
\[
 \sum_{T\in\mathcal S_b}P(T\mid b)
 =|\mathcal S_b|\,\frac{1}{C_b}=1.
\]
\end{proof}

\begin{proposition}[Uniform unary--binary shapes with root constraints]
\label{prop:generator_unary_uniform}
Fix $b,u\geq0$ and index the $K=2b+1$ skeleton positions in preorder, with position one at the root. Independently draw a uniform skeleton and a uniform nonnegative integer vector $(u_1,\ldots,u_K)$ summing to $u$, then insert $u_j$ unary nodes above position $j$. Every ordered shape with these arity counts has probability
\begin{equation}
 P(T\mid b,u)=\frac{1}{C_bW(u,K)},\qquad
 W(u,K)=\binom{u+K-1}{K-1}.
 \label{eq:generator_shape_probability}
\end{equation}
A prescribed unary root restricts the vector to $u_1\geq1$; a prescribed binary root requires $u_1=0$ and $b\geq1$. The respective counts are
\begin{equation}
 W_{\mathrm{unary}}=\binom{u+K-2}{K-1}\quad(u\geq1),\qquad
 W_{\mathrm{binary}}=\binom{u+K-2}{K-2}\quad(b\geq1).
 \label{eq:generator_root_compositions}
\end{equation}
A uniform draw among the admissible vectors assigns probability $1/(C_bW_r)$ to each shape with root arity $r$.
\end{proposition}
\begin{proof}
\textbf{One-to-one shape construction.}
Write
\[
 \mathcal U_{u,K}
 =\left\{(u_1,\ldots,u_K)\in\mathbb Z_{\geq0}^{K}:
          \sum_{j=1}^{K}u_j=u\right\},
\]
and let $\mathcal T_{b,u}$ be the target ordered shapes. Insertion defines
\[
 \Phi:\mathcal S_b\times\mathcal U_{u,K}\longrightarrow\mathcal T_{b,u},
 \qquad
 \Phi(S,\mathbf u)=\operatorname{insert}(S,\mathbf u).
\]
Every maximal unary chain lies above a unique surviving skeleton node, including the root. Contracting all such chains and recording their lengths in preorder therefore gives
\[
 \Phi^{-1}(T)=
 \bigl(\operatorname{contract}(T),\operatorname{chainlengths}(T)\bigr),
 \qquad
 \Phi^{-1}\circ\Phi=\mathrm{id},\quad
 \Phi\circ\Phi^{-1}=\mathrm{id}.
\]
Both operations preserve child order and the prescribed node counts.

\textbf{Uniform unary allocation.}
For $K>1$, the composition sampler selects $K-1$ separator positions uniformly:
\[
 1\leq t_1<\cdots<t_{K-1}\leq u+K-1.
\]
The forward and inverse maps are
\[
 \begin{aligned}
 u_1&=t_1-1,\qquad
 u_j=t_j-t_{j-1}-1\quad(2\leq j<K),\\
 u_K&=u+K-1-t_{K-1},\qquad
 t_j=j+\sum_{\ell=1}^{j}u_\ell\quad(1\leq j<K).
 \end{aligned}
\]
Hence
\[
 \begin{aligned}
 u_j&\geq0,\qquad \sum_{j=1}^{K}u_j=u,\\
 |\mathcal U_{u,K}|&=\binom{u+K-1}{K-1}=W(u,K),\qquad
 P(\mathbf u\mid u,K)=\frac{1}{W(u,K)}.
 \end{aligned}
\]
For $K=1$, $\mathcal U_{u,1}=\{(u)\}$ and $W(u,1)=1$. By the bijection and independence,
\[
 \begin{aligned}
 |\mathcal T_{b,u}|&=|\mathcal S_b|\,|\mathcal U_{u,K}|=C_bW(u,K),\\
 P(T\mid b,u)&=P(S\mid b)P(\mathbf u\mid u,K)
              =\frac{1}{C_b}\frac{1}{W(u,K)},\\
 \sum_{T\in\mathcal T_{b,u}}P(T\mid b,u)&=1.
 \end{aligned}
\]

\textbf{Conditioning on the root arity.}
For a unary root ($u\geq1$), reserving one root node gives the bijection
\[
 \begin{aligned}
 \{\mathbf u\in\mathcal U_{u,K}:u_1\geq1\}
 &\longleftrightarrow\mathcal U_{u-1,K},\\
 (u_1,u_2,\ldots,u_K)&\longmapsto(u_1-1,u_2,\ldots,u_K),\\
 W_{\mathrm{unary}}
 &=W(u-1,K)=\binom{u+K-2}{K-1}.
 \end{aligned}
\]
For a binary root ($b\geq1$), excluding the root position gives
\[
 \begin{aligned}
 \{\mathbf u\in\mathcal U_{u,K}:u_1=0\}
 &\longleftrightarrow\mathcal U_{u,K-1},\\
 (0,u_2,\ldots,u_K)&\longmapsto(u_2,\ldots,u_K),\\
 W_{\mathrm{binary}}
 &=W(u,K-1)=\binom{u+K-2}{K-2}.
 \end{aligned}
\]
The inverses add one to the first entry or prepend zero, respectively. Let $\mathcal T_{b,u,r}$ contain the shapes with root arity $r$. Each admissible vector has probability $1/W_r$, independently of the skeleton:
\[
 |\mathcal T_{b,u,r}|=C_bW_r,\qquad
 P(T\mid b,u,r)=\frac{1}{C_bW_r},\qquad
 \sum_{T\in\mathcal T_{b,u,r}}P(T\mid b,u,r)=1.
\]
The feasible boundary cases are
\[
 \begin{array}{c|c|c}
 (b,u)&\text{admissible allocation}&\text{result}\\ \hline
 b=0,\ u>0&(u)&\text{one unary chain}\\
 b>0,\ u=0&(0,\ldots,0)&\text{unchanged binary skeleton}\\
 b=u=0&(0)&\text{one leaf}.
 \end{array}
\]
A unary root with $u=0$ or a binary root with $b=0$ is infeasible and excluded.
\end{proof}

These statements characterize the raw, unlabeled shape proposal. The empirical draw of $(N,B)$ controls how often different arity classes occur; Catalan and composition weights make the draw uniform within the selected class and root constraint. Subsequent operator/leaf labels, simplification, and acceptance checks determine the distribution of the final formulas.

\begin{table}[!htbp]
\caption{Base operator probabilities after the unary floor and before forced assignments or contextual exclusions. \texttt{Neg}, \texttt{Inv}, \texttt{Pow2}, and \texttt{Pow3} denote negation, reciprocal, square, and cube.}
\label{tab:generator_operators}
\centering\small
\begin{tabular}{lr@{\qquad\qquad}lr}
\toprule
Binary operator & Probability & Unary operator & Probability\\
\midrule
\texttt{Add} & 0.219895438 & \texttt{Neg} & 0.034145627\\
\texttt{Sub} & 0.114883686 & \texttt{Inv} & 0.015434137\\
\texttt{Mul} & 0.509631078 & \texttt{Abs} & 0.012666583\\
\texttt{Div} & 0.135253958 & \texttt{Sqrt} & 0.117837361\\
\texttt{Pow} & 0.020335839 & \texttt{Pow2} & 0.278049997\\
 & & \texttt{Pow3} & 0.073586740\\
 & & \texttt{Exp} & 0.110216614\\
 & & \texttt{Log} & 0.032029858\\
 & & \texttt{Sin} & 0.200413219\\
 & & \texttt{Cos} & 0.113221225\\
 & & \texttt{Tan}, \texttt{Arcsin}, \texttt{Arccos} & 0.001000000 each\\
 & & \texttt{Arctan}, \texttt{Sinh}, \texttt{Cosh} & 0.001000000 each\\
 & & \texttt{Tanh} & 0.003398639\\
 & & \texttt{Cot}, \texttt{Sec}, \texttt{Csc} & 0.001000000 each\\
\bottomrule
\end{tabular}
\end{table}

\subsection{Variables, constants, and exponents}
\begin{enumerate}
\item \textbf{Variable coverage and ordinary leaves.}
Reserve exponent leaves, shuffle ordinary leaves, and place each of the $v$ variables once. Each remaining leaf is a constant with probability $0.702974962$, or a variable drawn uniformly from $\{x_1,\ldots,x_v\}$.

\item \textbf{Constant magnitudes and signs.}
Select a bin from Table~\ref{tab:generator_constants}, returning zero or sampling $\log_{10}|c|$ uniformly. With probability $0.151283414$, replace the nonzero draw by a uniformly chosen matching value from $\{0.5,1,2,e,\pi\}$, if available. Positive and negative signs have probabilities $0.667642615$ and $0.332357385$, respectively. Exclude zero for leaves directly under \texttt{Mul}, \texttt{Div}, \texttt{Pow}, \texttt{Inv}, or \texttt{Log}.

\item \textbf{General-power exponents.}
For a \texttt{Pow} leaf exponent, draw uniformly from $\{4,5,6\}$ with probability $0.536287857$, uniformly from $\{0.25,0.75,1.25,1.5,2.5,3.5\}$ with probability $0.062566906$, or from $\mathcal U(0.1,4)$ with probability $0.401145236$. Negative exponents and aliases $-1,0.5,2,3$ are excluded; reciprocal, square-root, square, and cube forms use unary operators.
\end{enumerate}

\begin{table}[!htbp]
\caption{Constant-magnitude proposals. The smallest nonzero bin has the finite sampling range $[10^{-100},10^{-12})$. The original mass above $10^6$ is redistributed proportionally between the two smallest nonzero bins.}
\label{tab:generator_constants}
\centering\small
\begin{tabular}{lr@{\qquad\qquad}lr}
\toprule
$|c|$ interval & Probability & $|c|$ interval & Probability\\
\midrule
$0$ & 0.000248984 & $[10^{-1},1)$ & 0.297092815\\
$(0,10^{-12})$ & 0.037271548 & $[1,10)$ & 0.313650602\\
$[10^{-12},10^{-6})$ & 0.062787186 & $[10,10^3)$ & 0.059584268\\
$[10^{-6},10^{-3})$ & 0.046988624 & $[10^3,10^6)$ & 0.011093352\\
$[10^{-3},10^{-1})$ & 0.171282619 & $[10^6,\infty)$ & 0\\
\bottomrule
\end{tabular}
\end{table}

\subsection{Simplification and acceptance checks}
\label{app:generator_filter}
\begin{enumerate}
\item \textbf{Processing order.}
Reparse serialized expressions, fold numerical subtrees bottom-up, and apply the whitelist until stable, rejecting after 16 unsuccessful passes. Process targets before pairing, random candidates after variable mapping, and edited candidates after editing, all before TED labeling.

\item \textbf{Numerical subtrees and domains.}
Require real, finite scalar constants with $|c|<10^6$. Reject zero denominators, \texttt{Inv}(0), $\log(c)$ for $c\leq0$, $\sqrt c$ for $c<0$, inverse sine/cosine for $|c|>1$, $0^c$ for $c\leq0$, and negative constant bases with noninteger powers. Folded fitted constants remain fitted; rewrite-introduced structural constants are not fitted.

\item \textbf{Whitelist of algebraic rewrites.}
The implemented whitelist contains the following groups:
\begin{itemize}
\item \emph{Identities and signs:} remove additive zeros, subtractive zero, multiplicative zero or one, and division by one; map $0-x$ to $-x$, multiplication or division by $-1$ to negation, $x-x$ to $0$, $x+x$ to $2x$, and $x x$ to $\mathrm{Pow2}(x)$. Remove double negation and double inversion, collapse repeated absolute value, and simplify negative numerical literals.
\item \emph{Collection of terms and factors:} flatten addition/multiplication components and combine constants, signed coefficients of like terms, and repeated factors. Group modulo associativity and commutativity, retaining first-occurrence order. Add positive integer exponents of repeated bases, without distribution or general fraction cancellation.
\item \emph{Powers and absolute values:} apply $x^0\mapsto1$, $x^1\mapsto x$, $1^x\mapsto1$, and aliases $(-1,0.5,2,3)\mapsto(\mathrm{Inv},\mathrm{Sqrt},\mathrm{Pow2},\mathrm{Pow3})$. Combine nested positive integer powers by multiplying exponents; map $\sqrt{x^2}$ to $|x|$, $|-x|$ to $|x|$, $|x^{2k}|$ to $x^{2k}$ for positive integer $k$, and $(-x)^2$ or $|x|^2$ to $x^2$.
\end{itemize}
The whitelist excludes general distribution, factorization, rational cancellation, trigonometric identities, exponential/logarithmic cancellation, and unrestricted real-power identities.

\item \textbf{Operator-nesting restrictions.}
The trigonometric family includes \texttt{Sin}, \texttt{Cos}, \texttt{Tan}, \texttt{Cot}, \texttt{Sec}, \texttt{Csc}, \texttt{Arcsin}, \texttt{Arccos}, \texttt{Arctan}, \texttt{Sinh}, \texttt{Cosh}, and \texttt{Tanh}. After folding, reject nested family members or any ancestor--descendant combination with \texttt{Exp}; also reject direct nesting within $\{\mathrm{Exp},\mathrm{Log}\}$ or $\{\mathrm{Pow},\mathrm{Pow2},\mathrm{Pow3},\mathrm{Sqrt}\}$. General powers require finite, positive numerical leaf exponents.

\item \textbf{Formula and retry limits.}
Require a variable, at most 100 nodes and 512 prefix tokens, candidate variables drawn from the target, and numerical validity. Retry limits are 200 structural attempts, 200 tokenization calls, 100 numerical-validity calls, and 100 pair attempts with fixed source/edit family. Archived pair retries may redraw target length; the current generator instead holds raw length fixed for up to 200 filtered attempts. Simplification can shorten either output.
\end{enumerate}

\subsection{Target--candidate construction and local edits}
\label{app:generator_pairs}
\begin{enumerate}
\item \textbf{Independent and edited pairs.}
Random pairs sample target and candidate independently; edited pairs modify a target copy. Map candidate and donor variables consistently into the target set, uniformly without replacement when possible and with replacement otherwise. Finite generation requests a 50:50 source split up to rounding; continuous pools are generated separately. Training uses the mixture in Appendix~\ref{app:generator_sampling}.

\item \textbf{Edit family, count, and total budget.}
Draw family/count proposals from Table~\ref{tab:edit_proposals}, truncate counts using Table~\ref{tab:edit_limits}, and renormalize. For target length $L$, total affected-node charges are capped at $\max(2,\lceil rL\rceil)$ using the listed $r$. Charges are one for point/unary edits, branch size plus connecting operator for branch edits, the larger subtree size for replacement, and removed nodes for promotion. These charges differ from TED.

\item \textbf{Location selection.}
Choose uniformly among eligible locations. With probability $0.85$, restrict to paths disjoint from earlier edits when available; otherwise use all eligible locations.
\end{enumerate}

\begin{table}[!htbp]
\caption{Edit-family probabilities and count proposals conditional on an edited pair. Short-target caps are applied after these proposals.}
\label{tab:edit_proposals}
\centering\small
\begin{tabular}{lrl}
\toprule
Family & Probability & Count values: probabilities\\
\midrule
Point relabeling & 0.20 & $1,2,3,4$: $(0.15,0.40,0.30,0.15)$\\
Unary insertion/deletion & 0.15 & $1,2,3,4$: $(0.20,0.40,0.25,0.15)$\\
Branch insertion/deletion & 0.05 & $1,2,3$: $(0.25,0.45,0.30)$\\
Subtree replacement & 0.15 & $1,2,3$: $(0.25,0.45,0.30)$\\
Crossover/promotion & 0.05 & $1,2$: $(0.60,0.40)$\\
Mixed families & 0.40 & $2,3,4$: $(0.10,0.30,0.60)$\\
\bottomrule
\end{tabular}
\end{table}

\begin{table}[!htbp]
\caption{Edit-count caps by target length $L$ and cumulative affected-node budget ratio $r$. Each piecewise list gives the largest permitted count.}
\label{tab:edit_limits}
\centering\small
\begin{tabular}{l p{0.64\linewidth} r}
\toprule
Family & Maximum count & $r$\\
\midrule
Point / unary & $1$ for $L\leq5$; $2$ for $6\leq L\leq20$; $3$ for $21\leq L\leq50$; $4$ otherwise & 0.30\\
Branch / subtree & $1$ for $L\leq7$; $2$ for $8\leq L\leq30$; $3$ otherwise & 0.35\\
Crossover/promotion & $1$ for $L\leq20$; $2$ otherwise & 0.20\\
Mixed & $2$ for $L\leq5$; $3$ for $6\leq L\leq15$; $4$ otherwise & 0.40\\
\bottomrule
\end{tabular}
\end{table}

\noindent\textbf{Operations within each family.}
These probabilities are conditional on the family. Operators follow Table~\ref{tab:generator_operators}, excluding the original operator for relabeling.

\begin{enumerate}
\item \textbf{Point relabeling.}
Choose numerical-value, variable, unary-operator, binary-operator, or variable-to-constant replacement with probabilities $(0,0.10,0.40,0.45,0.05)$. Variables change uniformly to another target variable. Operator changes preserve children and arity and require a variable-containing subtree. Variable-to-constant edits use the ordinary constant sampler and must retain a variable overall.

\item \textbf{Unary insertion or deletion.}
With equal probability $0.5$, wrap a subtree in a sampled unary operator or remove a unary operator while retaining its variable-containing child.

\item \textbf{Branch insertion or deletion.}
Choose insertion/deletion with probability $0.5$ each. For current size $\ell$, define $c_{10}(\ell)=\max(1,\min(10,\lceil0.10\ell\rceil))$. Insertion joins a subtree and a donor of at most $c_{10}(\ell)$ nodes under a sampled binary operator, with equally likely child order. Restrict donors to non-root subtrees with probability $0.70$. Deletion retains a variable-containing child and removes a sibling of at most $c_{10}(\ell)$ nodes. Enforce the 100-node limit.

\item \textbf{Subtree replacement.}
Select a subtree of size $s\leq c_{15}(\ell)$, with $c_{15}(\ell)=\max(1,\min(8,\lceil0.15\ell\rceil))$, excluding the root unless the tree has one node. Matched, shorter, and longer donors have probabilities $(0.70,0.05,0.25)$ and size intervals
\[
 \begin{split}
 &[\max(1,\lfloor0.8s\rfloor),\ \max(1,\lfloor0.8s\rfloor,\lceil1.2s\rceil)],\\
 &[1,\ \max(1,s-1)],\\
 &[s+1,\ \min(c_{15}(\ell)+2,s+4)].
 \end{split}
\]
Cap upper endpoints at $100-(\ell-s)$. Choose uniformly among eligible subtrees of a sampled donor formula, restricting to non-root donors with probability $0.70$. Both subtrees must contain a variable; reject identical replacements.

\item \textbf{Crossover or promotion.}
Crossover (probability $0.80$) uses the preceding cap and intervals, size-mode probabilities $(0.80,0.10,0.10)$, and non-root donors. Promotion (probability $0.20$) uniformly chooses an eligible ancestor--descendant pair, retains the variable-containing descendant, and removes 1 to $\max(1,\lceil0.20\ell\rceil)$ nodes. Apply the same disjoint-path preference.

\item \textbf{Mixed operations.}
For $m=2$, select two distinct families; for $m=3$, select two/three with probabilities $(0.75,0.25)$; for $m=4$, select two/three/four with probabilities $(0.60,0.30,0.10)$. Sample without replacement using point, unary, branch, subtree, and crossover/promotion weights $(0.30,0.25,0.20,0.20,0.05)$. Fill remaining slots with those weights, allowing repeats except for subtree and crossover/promotion, then shuffle.

\item \textbf{Retries, final validation, and labels.}
Allow 30 sequence attempts and 10 attempts per operation. Proposal routines allow 12 point, 8 unary, 12 branch, and 20 subtree trials; donor search allows 40 formulas. Retry failed pairs with the same edit family. Reject candidates that normalize back to the target, lose variables, violate constraints, or fail numerical validation. Then compute normalized multi-view TED (Appendix~\ref{app:distance}), rejecting nonfinite labels without targeting a distance interval. Command defaults allow a 40-second TED deadline, 200-second hard watchdog, and five failed task attempts before skipping.
\end{enumerate}

\subsection{Numerical observation sampling}
\label{app:generator_observations}
The archived run uses the base mixture sampler adapted from MDLformer. Let $v$ denote the number of target variables and $M$ the requested point count.
\begin{enumerate}
\item \textbf{Mixture parameters.}
Draw $C$ uniformly from $\{1,\ldots,9\}$, weights $a_k\sim\mathcal U(0,1)$, and $\rho_k=a_k/\sum_j a_j$. Each component draws $\boldsymbol\mu_k\sim\mathcal N(0,I_v)$, $c_{kj}\sim\mathcal U(0,1)$, and a random special-orthogonal rotation $R_k$ for $v>1$ (identity for $v=1$). Allocate points by $\operatorname{Multinomial}(M;\boldsymbol\rho)$.

\item \textbf{Component type and coordinates.}
Independently for each component, choose Gaussian or uniform sampling with probability $0.5$ each:
\[
 \mathbf x=\mathbf zR_k,\qquad
 \begin{cases}
 \mathbf z\sim\mathcal N(\boldsymbol\mu_k,\operatorname{diag}(\mathbf c_k)),
     &\text{Gaussian},\\
 z_j=\mu_{kj}+\sqrt{c_{kj}}\,u_j,\quad u_j\sim\mathcal U(-1,1),
     &\text{uniform}.
 \end{cases}
\]
Gaussian and uniform coordinate variances are $c_{kj}$ and $c_{kj}/3$, respectively. Rotate means and offsets together. The archived sampler leaves $X$ unstandardized and uses no affine augmentation.

\item \textbf{Evaluation, filtering, and repeated draws.}
Evaluate the target and discard nonfinite outputs or $|y|>9.999\times10^{100}$. Draw up to five fresh mixture batches until $M$ valid rows are available, then shuffle and retain $M$. Table construction allows 20 such calls and requires finite inputs/outputs. Failed stratified requests retry up to 32 rows from the same source/bin. Pair-generation checks use $M=256$.

\item \textbf{Inputs supplied to the verifier.}
Use 256 table rows, pad unused variable columns to 10, and normalize outputs by
\[
 \widetilde y_i=(y_i-\overline y)/(\operatorname{std}(y)+10^{-6}).
\]
Compute correlation summaries from 2,048 valid observations and use their first 256 shuffled rows for the table encoder. Each transformed correlation requires three valid points; mask invalid summaries (Appendix~\ref{app:training}). Observations come from the target; symbolic inputs come from the candidate.
\end{enumerate}

\subsection{TED-balanced training draws and validation}
\label{app:generator_sampling}
\begin{enumerate}
\item \textbf{Distance coverage and source probabilities.}
Draw each of 12 TED bins with probability $1/12$. Given bin $i$, choose edited or random pairs with probabilities $q_i$ and $1-q_i$, respectively (Table~\ref{tab:generator_ted_sampling}):
\[
 P(i,\mathrm{edit})=\frac{q_i}{12},\qquad
 P(i,\mathrm{random})=\frac{1-q_i}{12}.
\]
Marginal source probabilities are $0.45$ edited and $0.55$ random. Equal bin mass leaves within-bin distances distributed as in the pool; the last two bins are wider. The main text's $0.6$ boundary indicates the predominant source.

\item \textbf{How the mixture weights were fixed.}
The August 28 calibration scan contained 18,343,472 edited and 52,541,662 random rows, with valid TED labels throughout. For the four mixed-source bins, use counts $E_i,R_i$ to set
\[
 q_i=\frac{\lambda E_i}{\lambda E_i+R_i},
 \qquad \lambda=3.0244780430884326.
\]
Other bins use the table's zero/one choices. These calibration counts precede the final pool snapshot.

\item \textbf{Row selection, batches, and reproducibility.}
Draw bin/source requests independently. Traverse each source/bin by deterministic permutations without replacement, restarting when exhausted. Seed 42 and update number fix the global plan; replay category draws on resume. Use four workers, 32 examples per microbatch, and two accumulation rounds (global batch 256). Expected counts are $256/12$ per bin, 115.2 edited, and 140.8 random. Contrastive pretraining instead samples sources equally.

\item \textbf{Reserved rows and fixed validation set.}
Reserve the first 32 rows per source/bin from stratified training. Fixed validation contains 32 pairs per bin: 384 total, comprising 173 edited and 211 random pairs.
\end{enumerate}

\paragraph{Bayesian interpretation of the source choice.}
Let $Z\in\{\mathrm{edit},\mathrm{random}\}$ be the pair source and $I\in\{1,\ldots,12\}$ its TED bin. Let $E=\sum_iE_i>0$ and $R=\sum_iR_i>0$ be the calibration source totals. The empirical bin likelihoods are $P(I=i\mid Z=\mathrm{edit})=E_i/E$ and $P(I=i\mid Z=\mathrm{random})=R_i/R$.

\begin{proposition}[Bayesian source sampling and calibration across TED bins]
\label{prop:generator_ted_calibration}
For an edited-source prior $P(Z=\mathrm{edit})=\rho\in(0,1)$ and any bin of positive mixture probability, Bayes' rule gives
\begin{equation}
 P(\mathrm{edit}\mid i)
 =\frac{\rho E_i/E}{\rho E_i/E+(1-\rho)R_i/R}
 =\frac{\lambda E_i}{\lambda E_i+R_i},
 \qquad \lambda=\frac{\rho R}{(1-\rho)E}.
\label{eq:generator_ted_bayes}
\end{equation}
Conversely, any $\lambda>0$ corresponds to $\rho=\lambda E/(\lambda E+R)$.

Set $q_i=1$ for $i=1,2,3$, use this posterior for $i=4,5,6,7$, and set $q_i=0$ for $i=8,\ldots,12$. Drawing $I$ uniformly and then the source with edited probability $q_I$ produces
\begin{equation}
 \begin{aligned}
 Q(I=i)&=\frac{1}{12},\\
 Q(Z=\mathrm{edit})&=F(\lambda)
 =\frac{1}{12}\left(3+\sum_{i=4}^{7}\frac{\lambda E_i}{\lambda E_i+R_i}\right).
 \end{aligned}
\label{eq:generator_ted_marginal}
\end{equation}
If $E_i,R_i>0$ in the four mixed bins, every edited marginal in $(3/12,7/12)$, including $0.45$, determines a unique $\lambda>0$.
\end{proposition}
\begin{proof}
\textbf{Posterior source law.}
For a bin of positive mixture mass, the joint and marginal probabilities are
\[
 \begin{aligned}
 P(Z=\mathrm{edit},I=i)&=\rho E_i/E,\\
 P(Z=\mathrm{random},I=i)&=(1-\rho)R_i/R,\\
 P(I=i)&=\rho E_i/E+(1-\rho)R_i/R>0.
 \end{aligned}
\]
Therefore
\[
 \begin{aligned}
 P(\mathrm{edit}\mid i)
 &=\frac{P(Z=\mathrm{edit},I=i)}{P(I=i)}
   =\frac{\rho E_i/E}{\rho E_i/E+(1-\rho)R_i/R}\\
 &=\frac{[\rho R/((1-\rho)E)]E_i}
         {[\rho R/((1-\rho)E)]E_i+R_i}
   =\frac{\lambda E_i}{\lambda E_i+R_i},\\
 P(\mathrm{random}\mid i)&=\frac{R_i}{\lambda E_i+R_i},\qquad
 P(\mathrm{edit}\mid i)+P(\mathrm{random}\mid i)=1.
 \end{aligned}
\]
The prior--multiplier correspondence is
\[
 \lambda=\frac{\rho R}{(1-\rho)E}
 \ \Longleftrightarrow\
 \lambda E=\rho(\lambda E+R)
 \ \Longleftrightarrow\
 \rho=\frac{\lambda E}{\lambda E+R}\in(0,1),
 \qquad \lambda>0.
\]
In particular, the source-prior odds are $\rho/(1-\rho)=\lambda E/R$.

\textbf{Training marginals.}
Under uniform bin sampling and the prescribed $q_i$,
\[
 \begin{aligned}
 Q(I=i,Z=\mathrm{edit})&=\frac{q_i}{12},&
 Q(I=i,Z=\mathrm{random})&=\frac{1-q_i}{12},\\
 Q(I=i)&=\frac{q_i+(1-q_i)}{12}=\frac{1}{12},&
 \sum_{i=1}^{12}Q(I=i)&=1.
 \end{aligned}
\]
Thus
\[
 \begin{aligned}
 Q(Z=\mathrm{edit})
 &=\sum_{i=1}^{12}\frac{q_i}{12}\\
 &=\frac{1}{12}
   \left(\underbrace{\sum_{i=1}^{3}1}_{3}
   +\sum_{i=4}^{7}\frac{\lambda E_i}{\lambda E_i+R_i}
   +\underbrace{\sum_{i=8}^{12}0}_{0}\right)
 =F(\lambda).
 \end{aligned}
\]
The zero/one choices are imposed source restrictions; the Bayesian posterior is used only in the four mixed bins.

\textbf{Existence and uniqueness of calibration.}
With $E_i,R_i>0$ for $i=4,\ldots,7$,
\[
 \begin{aligned}
 \frac{d}{d\lambda}\frac{\lambda E_i}{\lambda E_i+R_i}
 &=\frac{E_i(\lambda E_i+R_i)-\lambda E_i^2}
         {(\lambda E_i+R_i)^2}
   =\frac{E_iR_i}{(\lambda E_i+R_i)^2}>0,\\
 F'(\lambda)&=\frac{1}{12}\sum_{i=4}^{7}
   \frac{E_iR_i}{(\lambda E_i+R_i)^2}>0.
 \end{aligned}
\]
Each denominator is positive on $(0,\infty)$, so $F$ is continuous there, and
\[
 \begin{aligned}
 \lim_{\lambda\downarrow0}\frac{\lambda E_i}{\lambda E_i+R_i}
 &=0,&
 \lim_{\lambda\to\infty}\frac{\lambda E_i}{\lambda E_i+R_i}
 &=1,\\
 \lim_{\lambda\downarrow0}F(\lambda)&=\frac{3}{12},&
 \lim_{\lambda\to\infty}F(\lambda)&=\frac{7}{12}.
 \end{aligned}
\]
For any target $\tau\in(3/12,7/12)$, these limits provide
\[
 0<\lambda_-<\lambda_+<\infty,\qquad
 F(\lambda_-)<\tau<F(\lambda_+).
\]
Continuity gives a solution in $(\lambda_-,\lambda_+)$, while
\[
 \lambda_1<\lambda_2\ \Longrightarrow\ F(\lambda_1)<F(\lambda_2)
\]
gives uniqueness. Therefore
\[
 \forall\tau\in(3/12,7/12),\quad
 \exists!\lambda>0:\ F(\lambda)=\tau.
\]
Since $3/12<0.45<7/12$, the reported target has a unique multiplier.
\end{proof}
For the reported calibration counts, $\lambda\simeq3.0244780430884326$ gives the target edited marginal $0.45$. This final marginal follows uniform bin sampling and need not equal the auxiliary source prior $\rho$.

\begin{table}[!htbp]
\caption{TED-bin probabilities, calibrated source choices, and the counts used to determine them. Every bin has training probability $1/12$; counts describe the August 28 calibration scan.}
\label{tab:generator_ted_sampling}
\centering\small
\begin{tabular}{lrrr r}
\toprule
TED interval & Edited count $E_i$ & Random count $R_i$ & $q_i$ (edit) & $1-q_i$ (random)\\
\midrule
$[0,0.1)$   & 1,562,156 & 106,540 & 1 & 0\\
$[0.1,0.2)$ & 2,578,355 & 148,986 & 1 & 0\\
$[0.2,0.3)$ & 3,802,093 & 478,678 & 1 & 0\\
$[0.3,0.4)$ & 3,536,928 & 1,073,557 & 0.908795814 & 0.091204186\\
$[0.4,0.5)$ & 2,699,071 & 2,321,306 & 0.778598241 & 0.221401759\\
$[0.5,0.6)$ & 2,159,613 & 6,491,726 & 0.501534776 & 0.498465224\\
$[0.6,0.7)$ & 1,147,026 & 12,966,794 & 0.211071169 & 0.788928831\\
$[0.7,0.8)$ & 405,094 & 9,157,137 & 0 & 1\\
$[0.8,0.9)$ & 216,774 & 4,606,913 & 0 & 1\\
$[0.9,1)$   & 37,601 & 1,838,357 & 0 & 1\\
$[1,1.5)$   & 163,252 & 7,710,484 & 0 & 1\\
$[1.5,\infty)$ & 35,509 & 5,641,184 & 0 & 1\\
\bottomrule
\end{tabular}
\end{table}

\section{Verifier architecture and training}
\label{app:training}
\subsection{Inputs and regression features}
The interface supports up to ten variables, at most 512 formula tokens, and at least three finite observations. It samples up to 256 rows for the table encoder and 2,048 rows for regression summaries. The table branch standardizes $y$ by $(y-\operatorname{mean}(y))/(\operatorname{std}(y)+10^{-6})$. Regression summaries use raw values and their specified transforms.

For each $x_j$, the 12 target/input transform pairs are
\begin{align}
(\log y,\log x_j),\quad &(y,x_j),\ (y,x_j^2),\ (y,x_j^3),\
(y,\sqrt{x_j}),\ (y,\log x_j),\nonumber\\
&(y,x_j^{-1}),\ (y,x_j^{-2}),\ (y,e^{x_j}),\
(y,\sin x_j),\ (y,\cos x_j),\ (y,\tan x_j).
\end{align}
Write the $k$th pair as $(g_k(y),h_k(x_j))$. For each variable $j$, let $\mathcal I_{jk}$ contain the rows on which both transformed values are finite and within their domains; the log--log pair, for example, requires $x_{ij}>0$ and $y_i>0$. With $u_i=h_k(x_{ij})$ and $v_i=g_k(y_i)$, we compute the signed Pearson correlation
\[
r_{jk}=\frac{\sum_{i\in\mathcal I_{jk}}(u_i-\bar u)(v_i-\bar v)}
{\sqrt{\sum_{i\in\mathcal I_{jk}}(u_i-\bar u)^2}\sqrt{\sum_{i\in\mathcal I_{jk}}(v_i-\bar v)^2}},
\]
where the means are over $\mathcal I_{jk}$. This equals $\operatorname{sign}(\widehat\beta_{jk})\sqrt{R_{jk}^2}$ for a one-predictor least-squares regression with an intercept, where $\widehat\beta_{jk}$ is its slope and $R_{jk}^2$ its coefficient of determination. For numerical stability, each transformed vector is divided by its largest absolute value before centering, which leaves its correlation unchanged, and the result is clipped to $[-1,1]$. A relation is masked if fewer than three valid rows remain, either transformed vector has zero variance, or the correlation is undefined; unused variables are also masked.

Each valid $(j,k)$ relation produces a token. Let $e_{\mathrm{out}}(g_k)$, $e_{\mathrm{pair}}(k)$, and $e_{\mathrm{var}}(j)$ be learned 64-dimensional embeddings for the output form ($y$ or $\log y$), the relation identity, and the input variable. The 12 relation identities are distinct, including the two relations that use $\log x_j$. The descriptor and gated token are
\[
\begin{aligned}
q_{jk}&=[e_{\mathrm{out}}(g_k);e_{\mathrm{pair}}(k);e_{\mathrm{var}}(j);r_{jk}]\in\mathbb R^{193},\\
t_{jk}&=|r_{jk}|(W_Rq_{jk}+b_R)\in\mathbb R^{768}.
\end{aligned}
\]
The projection $W_R,b_R$ is shared across relations. The signed $r_{jk}$ enters the descriptor, while $|r_{jk}|$ reduces the magnitude of weak-correlation tokens. Masked tokens are excluded from the learned attention pooling. For the valid set $\mathcal V$, the pooled regression representation is
\[
\alpha_{jk}=\frac{\exp(w^\top t_{jk}+b)}{\sum_{(p,q)\in\mathcal V}\exp(w^\top t_{pq}+b)},\qquad
z_R=\sum_{(j,k)\in\mathcal V}\alpha_{jk}t_{jk}.
\]
Here $w,b$ are learned attention-score parameters. If no relation is valid, the pooled vector is zero.

The data encoder embeds sign/mantissa/exponent tokens into 64 dimensions, concatenates each padded row, and projects it to width 768. The formula encoder adds token, tree-address, and depth embeddings. Both use six Transformer layers, six attention heads, learned attention pooling, and output projection. SwiGLU hidden width is 2,048. Part of its value bias is tied to match the parameter count of the original four-times-width feed-forward block.

\subsection{Optimization}
For matched Stage~1 data/formula embeddings $(u_i,v_i)$, the loss is
\begin{equation}
\mathcal L_{\mathrm{align}}=-\frac{1}{2B}\sum_i
\left[\log\frac{\exp(u_i^\top v_i/t)}{\sum_j\exp(u_i^\top v_j/t)}
+\log\frac{\exp(v_i^\top u_i/t)}{\sum_j\exp(v_i^\top u_j/t)}\right],
\end{equation}
with temperature $t=1$ and raw, unnormalized dot products. The paired expression is the generating target. Stage~2 initializes both encoders, adds the regression branch and fusion head, and updates all parameters using Equation~\ref{eq:loss}. The regression head does not clip outputs.

\begin{table}[!htbp]
\caption{Completed verifier training configuration. Checkpoints were selected on synthetic validation, not benchmark test scores.}
\label{tab:training}
\centering
\small
\begin{tabular}{lll}
\toprule
Setting & Stage 1 & Stage 2\\
\midrule
Global batch size & 256 & 256\\
Updates & 55,000 & 100,000\\
Learning rate & $10^{-4}$ peak & $2{\times}10^{-5}$ to $2{\times}10^{-6}$\\
Schedule & Warmup--stable--decay & Warmup + cosine decay\\
Warmup updates & 2,000 & 2,000\\
Stable / decay updates & 48,000 / 5,000 & ---\\
Seed / training GPUs & 42 / 4 & 42 / 4\\
Selected checkpoint update & 54,500 & 86,600\\
\bottomrule
\end{tabular}
\end{table}
Stage~2 uses AdamW with $\beta=(0.9,0.95)$, weight decay 0.01, gradient clipping at 1.0, and microbatches of 32. Validation is evaluated every 100 updates. The selected Stage~2 checkpoint has archived validation objective 0.12612656.

The bins are $[0,0.1),\ldots,[0.9,1.0),[1.0,1.5),[1.5,\infty)$, sampled equally. Conditional probabilities of drawing an edited pair are
\[
(1,\ 1,\ 1,\ 0.9088,\ 0.7786,\ 0.5015,\ 0.2111,\ 0,\ 0,\ 0,\ 0,\ 0).
\]
Their mean is 0.45, with complementary random-pair probabilities. Validation draws a fixed source mixture with 32 examples per bin, totaling 384 examples. The reserved source/bin rows and final mixed validation batch are distinct.

\begin{figure}[!htbp]
\centering
\includegraphics[width=0.94\linewidth]{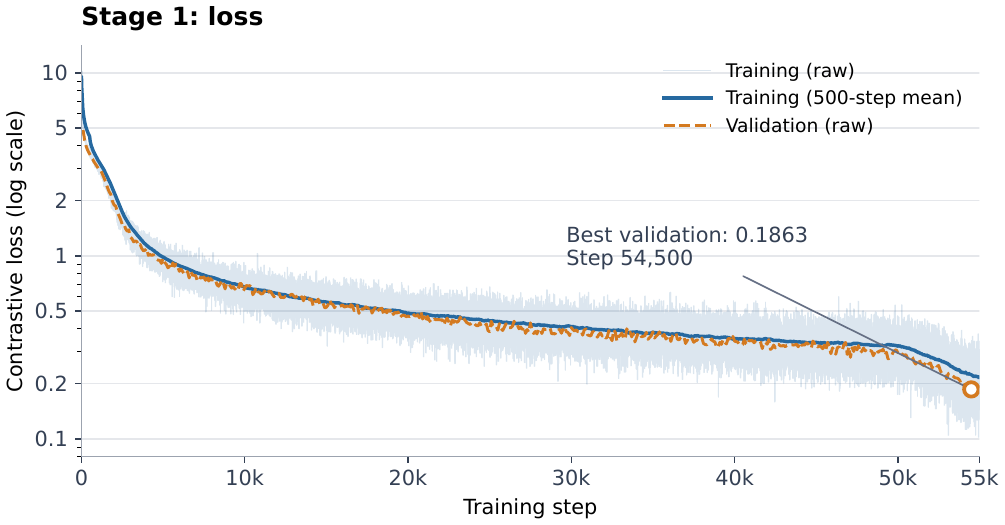}
\caption{Stage~1 training and validation loss.}
\label{fig:verifier_stage1_loss}
\end{figure}

\begin{figure}[!htbp]
\centering
\includegraphics[width=0.94\linewidth]{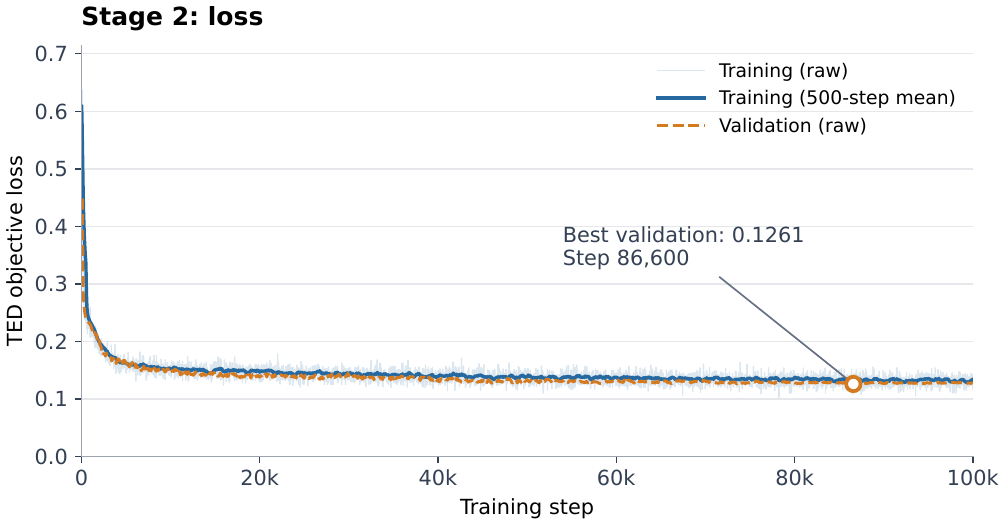}
\caption{Stage~2 training and validation loss.}
\label{fig:verifier_stage2_loss}
\end{figure}

\begin{figure}[!htbp]
\centering
\begin{minipage}[t]{0.49\linewidth}
\centering
\includegraphics[width=\linewidth]{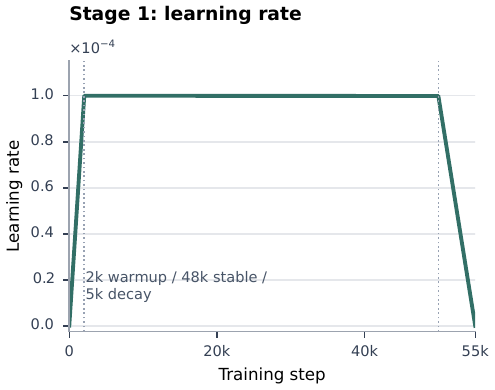}
\caption{Stage~1 learning-rate schedule.}
\label{fig:verifier_stage1_lr}
\end{minipage}\hfill
\begin{minipage}[t]{0.49\linewidth}
\centering
\includegraphics[width=\linewidth]{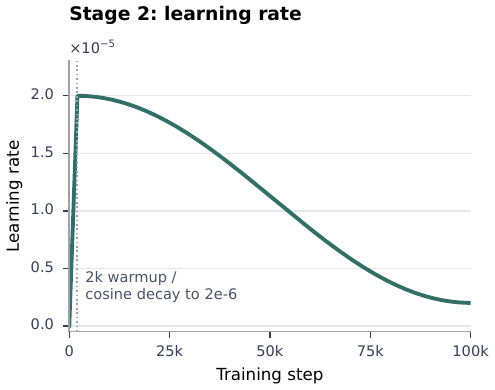}
\caption{Stage~2 learning-rate schedule.}
\label{fig:verifier_stage2_lr}
\end{minipage}
\end{figure}

\section{Search implementation and ablation controls}
\label{app:search}

\begingroup
\setlength{\intextsep}{4pt}
\begin{algorithm}[H]
    \footnotesize
    \caption{L2G-SR: symbolic search with learned structural feedback}
    \label{alg:search}
    \begin{algorithmic}[1]
    \Require Observations $\mathcal D$, metadata, frozen $\vpred$; proposal budget
    \State Split $\mathcal D$; initialize four islands
    \For{each round in the balanced M/T schedule}
      \State Select island and parent by validation MSE or predicted TED
      \State Analyst uses diagnostics and history; generator proposes four formulas
      \State For each proposal: validate, fit, simplify, refit, and score by MSE and $\vpred$
      \State Add valid formulas; deduplicate and prune by both scores
    \EndFor
    \State \Return simplest valid formula within the validation-MSE tolerance
\end{algorithmic}
\end{algorithm}
\endgroup

\begin{table}[!htbp]
\caption{Search settings shared by TED and matched numerical-only conditions.}
\centering
\begin{tabular}{ll}
\toprule
Setting & Value\\
\midrule
Backbone for both roles / temperature & Llama3.1-8B-Instruct / 0.8\\
Islands / capacity / M and T quotas & 4 / 64 / 32 each\\
Initialization / rounds / children & 40 / 90 / 4\\
M rounds / T rounds & 45 / 45\\
Tournament sizes, M / T & 3 / 2, with replacement\\
Credibility threshold / exploration cap & Validation NMSE $\leq1$ / 8\\
Validation fraction / maximum fit rows & 0.2 / 4,096\\
Independent constants / BFGS starts & At most 10 / 3\\
BFGS gradient tolerance & $10^{-5}$\\
Per-slot numerical budget & 30 seconds; 6,000 objective evaluations\\
Postprocessing passes & At most 16\\
Analyzer tool calls / interactions & At most 8 / 4 per round\\
Analyzer / generator input-token caps & 6,000 / 4,000\\
Analyzer / generator output-token caps & 2,200 / 1,800\\
\bottomrule
\end{tabular}
\end{table}

When subtraining exceeds 4,096 rows, fitting observations are sampled by target-value strata. Constants are unbounded real values; the fitting budget bounds computation, not coefficients. Mapped parent parameters provide warm starts. Linear-in-parameter formulas can use a scaled least-squares initialization. Supported symbolic derivatives supply gradients, with numerical differentiation as a fallback.

A ratio above $10^4$ between the 95th and 5th percentiles of $|y|$, or its maximum and median, activates transformed fitting. Positive targets use log residuals; otherwise the implementation uses $\operatorname{asinh}(y/s_y)$. To screen term removal before refitting, we set each fitted constant $\theta_j$ to zero in the current formula $f$, giving $f_{j\to0}$. We consider the contribution small only when, over all $n$ search observations,
\begin{equation}
\begin{aligned}
\Delta_i^{(j)} &= \frac{f_{j\to0}(x_i)-f(x_i)}{\max\{|y_i|,10^{-6}s_y\}},\\
\sqrt{\frac{1}{n}\sum_{i=1}^n\bigl(\Delta_i^{(j)}\bigr)^2} &\le 10^{-4},\qquad \max_i\bigl|\Delta_i^{(j)}\bigr|\le 10^{-3},
\end{aligned}
\label{eq:contribution}
\end{equation}
where $s_y=\max\{\operatorname{median}_i|y_i|,\,10^{-6}Q_{0.75}(|y|),\,\epsilon_{\mathrm{tiny}}\}$ is the robust target scale. Other proposed edits use the same normalized output change before refitting. After refitting, the proposal must satisfy, separately on fixed subtraining and validation,
\[
\mse_{\mathrm{new}}\leq
\mse_{\mathrm{old}}+10^{-3}\mse_{\mathrm{old}}+
10^{-12}\operatorname{Var}(y).
\]
These checks are untrimmed. Symbolic-equivalence proposals use stricter output-change checks. The verifier scores only the final accepted effective expression.

Let $m$ be the best validation MSE and $\sigma_y^2$ the target variance used by final selection. Eligible candidates satisfy
\[
\mse(f)\leq m+\max\{10^{-10}\max(\sigma_y^2,10^{-30}),10^{-4}m\}.
\]
Among them, selection minimizes complexity, parameter count, MSE, and then a stable identifier. Tournament ties are broken uniformly. Unavailable structural predictions trigger a random-valid-candidate fallback for the T tournament.

The no-TED control substitutes MSE in the compatibility score field used by comparisons and retention, and labels this substitution in the prompts. The schedule and tournament sizes remain unchanged. It performs no TED model loading or prediction; any offline true-distance computation cannot affect search decisions.

\section{Detailed completed results}
\label{app:results}
\subsection{Aggregation and reporting}
We aggregate completed equations within each domain/seed and then report
\[
\bar m=\frac13\sum_{s=1}^3m_s,\qquad
\operatorname{SD}(m)=\sqrt{\frac12\sum_{s=1}^3(m_s-\bar m)^2}.
\]
Accuracy dispersion is in percentage points. Entries taken from DE preserve its reported precision and aggregation. Our RESTART reproduction reports sample SD across its own three seeds.

The task configuration totals 240 equations. We report the actual evaluated domain counts.

The float64 evaluator discards $\lfloor0.05n\rfloor$ points per equation. Nonfinite predictions consume this allowance before finite points are ordered by relative error. Exceeding the allowance makes the equation invalid. The retained-set target mean is recomputed for the NMSE denominator, protected by the smallest positive float64 value. This differs from untrimmed online validation.

\begin{table}[!htbp]
\caption{Sample standard deviations for the RESTART and SR-Scientist means in Table~\ref{tab:main}. Accuracy standard deviations are in percentage points.}
\label{tab:baseline_sd}
\centering
\scriptsize
\setlength{\tabcolsep}{3pt}
\resizebox{\linewidth}{!}{%
\begin{tabular}{@{}l*{5}{cc}@{}}
\toprule
 & \multicolumn{2}{c}{Transform} & \multicolumn{2}{c}{Physics} & \multicolumn{2}{c}{Materials} & \multicolumn{2}{c}{Chemistry} & \multicolumn{2}{c}{Biology}\\
Method & NMSE SD & $\acc$ SD & NMSE SD & $\acc$ SD & NMSE SD & $\acc$ SD & NMSE SD & $\acc$ SD & NMSE SD & $\acc$ SD\\
\midrule
\multicolumn{11}{c}{\textbf{Llama3.1-8B-Instruct}}\\
RESTART & 7.06e-3 & 1.88 & 3.47e-3 & 4.73 & 6.91e-5 & 6.11 & 1.23e-6 & 3.21 & 1.91e-2 & 6.36\\
SR-Scientist & 1.94e-3 & 1.38 & 3.18e-3 & 1.31 & 1.49e-3 & 0.00 & 7.47e-5 & 2.78 & 4.53e-2 & 4.17\\
\midrule
\multicolumn{11}{c}{\textbf{Qwen3-4B-Instruct-2507}}\\
RESTART & 1.08e2 & 1.04 & 5.48e-2 & 2.62 & 1.14e-3 & 4.00 & 3.81e-2 & 8.49 & 2.61e0 & 2.41\\
\bottomrule
\end{tabular}%
}
\end{table}

Table~\ref{tab:llm_usage} compares our reproduced RESTART and SR-Scientist runs with \ours{} on successful LLM usage. We count distinct successful API responses, including responses whose generated expressions are later rejected, and exclude failed requests. Input and output tokens are the provider-reported prompt and completion tokens for those responses. Each mean uses the same 480 completed problem runs per method; duplicate response IDs are removed from the RESTART and \ours{} journals, while SR-Scientist uses its successful-response counters.

\begin{table}[!htbp]
\caption{Successful LLM calls and tokens per problem for the Llama3.1-8B-Instruct runs shared by the three methods.}
\label{tab:llm_usage}
\centering
\small
\setlength{\tabcolsep}{5pt}
\begin{tabular}{@{}lrrr@{}}
\toprule
Method & \shortstack{Successful calls\\per problem} & \shortstack{Input tokens\\per problem} & \shortstack{Output tokens\\per problem}\\
\midrule
RESTART (reproduced) & 422.7 & 715,609 & 103,528\\
SR-Scientist (reproduced) & 519.2 & 1,517,007 & 63,266\\
\ours{} & 189.2 & 847,822 & 48,130\\
\bottomrule
\end{tabular}
\end{table}

\subsection{In-distribution Oscillation results}
\label{app:oscillation_id}
Table~\ref{tab:oscillation_id} reports the in-distribution results on the two real-world Oscillation problems.
\begin{table}[!htbp]
\centering
\centering
\scriptsize
\setlength{\tabcolsep}{1.5pt}
\renewcommand{\arraystretch}{1.15}
\newcommand{\oscmeanstd}[2]{\begin{tabular}[t]{@{}c@{}}\strut#1\\{\tiny\strut\textpm#2}\end{tabular}}
\def\osccolwidth{\dimexpr0.2\linewidth-1.6\tabcolsep\relax}
\begin{tabular}{@{}p{\osccolwidth}*{4}{>{\centering\arraybackslash}p{\osccolwidth}}@{}}
\toprule
 & \multicolumn{2}{c}{\textbf{Oscillation 1}} & \multicolumn{2}{c}{\textbf{Oscillation 2}}\\
\cmidrule(lr){2-3}\cmidrule(lr){4-5}
\textbf{Method} & NMSE $\downarrow$ & $\acc^{\mathrm{pt}}$ $\uparrow$ & NMSE $\downarrow$ & $\acc^{\mathrm{pt}}$ $\uparrow$\\
\midrule
\rowcolor{gray!10}
\multicolumn{5}{c}{\textbf{Llama3.1-8B-Instruct}}\\
LLMDirect & \oscmeanstd{6.38e-7}{4.34e-7} & \oscmeanstd{\textbf{\boldmath 100.00}}{0.00} & \oscmeanstd{\underline{2.31e-5}}{1.89e-5} & \oscmeanstd{\underline{88.79}}{9.37}\\
LLM-SR & \oscmeanstd{1.21e-5}{1.67e-5} & \oscmeanstd{94.31}{9.86} & \oscmeanstd{4.51e-4}{7.17e-4} & \oscmeanstd{63.97}{34.95}\\
LaSR & \oscmeanstd{9.61e-4}{3.66e-4} & \oscmeanstd{16.26}{3.35} & \oscmeanstd{8.36e-4}{3.27e-4} & \oscmeanstd{31.37}{9.05}\\
SGA & \oscmeanstd{2.23e-4}{3.76e-4} & \oscmeanstd{68.78}{44.14} & \oscmeanstd{4.74e-3}{2.15e-3} & \oscmeanstd{7.96}{1.71}\\
RESTART & \oscmeanstd{4.61e-6}{3.27e-6} & \oscmeanstd{\underline{99.70}}{0.52} & \oscmeanstd{4.11e-3}{1.66e-3} & \oscmeanstd{11.14}{1.28}\\
SR-Scientist & \oscmeanstd{\underline{8.80e-10}}{7.63e-10} & \oscmeanstd{\textbf{\boldmath 100.00}}{0.00} & \oscmeanstd{4.21e-4}{7.30e-4} & \oscmeanstd{74.55}{44.08}\\
\rowcolor{blue!4}
\ours & \oscmeanstd{\textbf{\boldmath 6.22e-30}}{7.03e-30} & \oscmeanstd{\textbf{\boldmath 100.00}}{0.00} & \oscmeanstd{\textbf{\boldmath 3.61e-7}}{5.21e-7} & \oscmeanstd{\textbf{\boldmath 100.00}}{0.00}\\
\bottomrule
\end{tabular}
\caption{ID results on Oscillation 1 and 2.}
\label{tab:oscillation_id}

\end{table}

\subsection{Out-of-distribution evaluation}
\label{app:ood}
We evaluate how the selected expressions generalize to shifted inputs.

\begin{table}[!htbp]
\caption{OOD evaluation on LSR-Synth. Published baseline rows are retained as reported. Our reproduced rows show mean $\pm$ standard deviation across three seeds.}
\label{tab:ood}
\centering
\scriptsize
\setlength{\tabcolsep}{2pt}
\renewcommand{\arraystretch}{1.15}
\newcommand{\oodmeanstd}[2]{\begin{tabular}[t]{@{}c@{}}\strut#1\\{\tiny\strut\textpm#2}\end{tabular}}
\resizebox{\linewidth}{!}{%
\begin{tabular}{@{}l*{4}{cc}@{}}
\toprule
 & \multicolumn{2}{c}{\textbf{Physics}} & \multicolumn{2}{c}{\textbf{Materials}} & \multicolumn{2}{c}{\textbf{Chemistry}} & \multicolumn{2}{c}{\textbf{Biology}}\\
\cmidrule(lr){2-3}\cmidrule(lr){4-5}\cmidrule(lr){6-7}\cmidrule(lr){8-9}
\textbf{Method} & NMSE $\downarrow$ & $\acc$ $\uparrow$ & NMSE $\downarrow$ & $\acc$ $\uparrow$ & NMSE $\downarrow$ & $\acc$ $\uparrow$ & NMSE $\downarrow$ & $\acc$ $\uparrow$\\
\midrule
\rowcolor{gray!10}
\multicolumn{9}{c}{\textbf{Llama3.1-8B-Instruct}}\\
A-SR & 6.55e3 & \underline{25.00} & \underline{6.90e-4} & \textbf{\boldmath 100.00} & 2.06e0 & 25.00 & \underline{1.08e0} & \underline{45.83}\\
RESTART & \oodmeanstd{2.13e2}{1.89e2} & \oodmeanstd{10.08}{4.84} & \oodmeanstd{6.65e-3}{5.61e-3} & \oodmeanstd{\underline{90.28}}{4.81} & \oodmeanstd{\underline{7.79e-1}}{2.44e-1} & \oodmeanstd{\underline{41.90}}{4.36} & \oodmeanstd{2.10e1}{2.26e1} & \oodmeanstd{18.84}{6.64}\\
SR-Scientist & \oodmeanstd{\underline{7.70e1}}{7.36e1} & \oodmeanstd{21.71}{2.69} & \oodmeanstd{9.09e-2}{7.82e-2} & \oodmeanstd{75.00}{11.02} & \oodmeanstd{2.33e0}{2.53e0} & \oodmeanstd{35.24}{4.36} & \oodmeanstd{7.49e3}{1.30e4} & \oodmeanstd{23.19}{2.51}\\
\rowcolor{blue!4}
\ours & \oodmeanstd{\textbf{\boldmath 2.13e-3}}{1.42e-3} & \oodmeanstd{\textbf{\boldmath 90.70}}{2.33} & \oodmeanstd{\textbf{\boldmath 2.27e-5}}{2.21e-5} & \oodmeanstd{\textbf{\boldmath 100.00}}{0.00} & \oodmeanstd{\textbf{\boldmath 1.92e-1}}{8.96e-2} & \oodmeanstd{\textbf{\boldmath 69.22}}{7.27} & \oodmeanstd{\textbf{\boldmath 1.99e-1}}{1.56e-1} & \oodmeanstd{\textbf{\boldmath 66.67}}{5.02}\\
\midrule
\rowcolor{gray!10}
\multicolumn{9}{c}{\textbf{Qwen3-4B-Instruct-2507}}\\
LLMDirect & 3.43e4 & 6.82 & 1.94e-1 & 72.00 & 5.98e6 & 0.00 & 1.18e2 & 4.17\\
LLM-SR & 8.17e4 & 9.09 & 2.75e-1 & 72.00 & 6.57e2 & 11.11 & 2.24e1 & 12.50\\
LaSR & 1.08e4 & 9.09 & 3.25e-1 & 40.00 & 4.49e2 & 2.78 & 9.66e2 & 4.17\\
SGA & 9.92e5 & 6.82 & 2.91e-1 & 40.00 & 1.55e3 & 2.78 & 1.83e5 & 4.17\\
DE & 1.97e3 & 15.91 & 4.85e-2 & 80.00 & 1.09e1 & 11.11 & 1.14e1 & 12.50\\
A-SR & \underline{3.17e2} & \underline{18.18} & \underline{3.64e-3} & \underline{96.00} & \underline{8.49e-1} & \underline{19.44} & \underline{7.41e0} & \underline{16.67}\\
\rowcolor{blue!4}
\ours & \oodmeanstd{\textbf{\boldmath 1.06e-2}}{1.37e-2} & \oodmeanstd{\textbf{\boldmath 89.15}}{1.34} & \oodmeanstd{\textbf{\boldmath 1.51e-5}}{1.89e-5} & \oodmeanstd{\textbf{\boldmath 100.00}}{0.00} & \oodmeanstd{\textbf{\boldmath 1.41e-1}}{4.11e-2} & \oodmeanstd{\textbf{\boldmath 63.81}}{1.65} & \oodmeanstd{\textbf{\boldmath 1.08e-1}}{3.99e-2} & \oodmeanstd{\textbf{\boldmath 59.42}}{2.51}\\
\bottomrule
\end{tabular}%
}
\end{table}

Under this protocol, \ours{} has the best displayed NMSE and accuracy in every domain for both backbones.

\subsection{Oscillation out-of-distribution results}
\label{app:oscillation_ood}
Table~\ref{tab:oscillation_ood} reports the OOD results corresponding to the ID comparison in Table~\ref{tab:oscillation_id}. We evaluate the same selected formulas without refitting on the 10,000 observations in each official OOD split. Metrics are scored in float64 after removing the 500 observations with the largest relative errors. Pointwise $\acc^{\mathrm{pt}}$ is the percentage of the 9,500 retained observations whose relative error is at most 0.01. All entries are means and sample standard deviations across three seeds.

\begin{table}[!htbp]
\caption{OOD results on Oscillation 1 and 2. NMSE ($\downarrow$) and pointwise $\acc^{\mathrm{pt}}$ ($\uparrow$, \%) are computed after 5\% relative-error trimming. Entries report mean $\pm$ sample standard deviation over three seeds. Bold and underlined means indicate the best and second-best displayed values, respectively.}
\label{tab:oscillation_ood}
\centering
\scriptsize
\setlength{\tabcolsep}{1.5pt}
\renewcommand{\arraystretch}{1.15}
\newcommand{\oscmeanstd}[2]{\begin{tabular}[t]{@{}c@{}}\strut#1\\{\tiny\strut\textpm#2}\end{tabular}}
\def\osccolwidth{\dimexpr0.2\linewidth-1.6\tabcolsep\relax}
\begin{tabular}{@{}p{\osccolwidth}*{4}{>{\centering\arraybackslash}p{\osccolwidth}}@{}}
\toprule
 & \multicolumn{2}{c}{\textbf{Oscillation 1}} & \multicolumn{2}{c}{\textbf{Oscillation 2}}\\
\cmidrule(lr){2-3}\cmidrule(lr){4-5}
\textbf{Method} & NMSE $\downarrow$ & $\acc^{\mathrm{pt}}$ $\uparrow$ & NMSE $\downarrow$ & $\acc^{\mathrm{pt}}$ $\uparrow$\\
\midrule
\rowcolor{gray!10}
\multicolumn{5}{c}{\textbf{Llama3.1-8B-Instruct}}\\
LLMDirect & \oscmeanstd{1.03e-2}{0.88e-2} & \oscmeanstd{45.06}{18.89} & \oscmeanstd{\underline{3.83e-3}}{3.01e-3} & \oscmeanstd{52.51}{14.82}\\
LLM-SR & \oscmeanstd{3.95e-2}{4.49e-2} & \oscmeanstd{23.52}{33.64} & \oscmeanstd{6.28e-3}{0.76e-3} & \oscmeanstd{29.80}{18.29}\\
LaSR & \oscmeanstd{1.98e-1}{2.44e-1} & \oscmeanstd{4.60}{4.91} & \oscmeanstd{1.75e-2}{1.56e-2} & \oscmeanstd{13.99}{5.65}\\
SGA & \oscmeanstd{3.19e-1}{2.78e-1} & \oscmeanstd{37.22}{50.88} & \oscmeanstd{2.30e-2}{1.46e-2} & \oscmeanstd{7.12}{3.86}\\
RESTART & \oscmeanstd{8.51e-3}{6.89e-3} & \oscmeanstd{15.10}{19.34} & \oscmeanstd{2.05e20}{3.55e20} & \oscmeanstd{1.15}{1.33}\\
SR-Scientist & \oscmeanstd{\underline{3.88e-4}}{4.14e-4} & \oscmeanstd{\underline{79.11}}{15.35} & \oscmeanstd{2.13e-2}{3.68e-2} & \oscmeanstd{\underline{65.70}}{55.48}\\
\rowcolor{blue!4}
\ours & \oscmeanstd{\textbf{\boldmath 3.11e-29}}{3.48e-29} & \oscmeanstd{\textbf{\boldmath 100.00}}{0.00} & \oscmeanstd{\textbf{\boldmath 5.95e-5}}{5.55e-5} & \oscmeanstd{\textbf{\boldmath 91.24}}{9.41}\\
\bottomrule
\end{tabular}
\end{table}

\ours{} has the lowest mean OOD NMSE and highest mean pointwise accuracy on both problems. On Oscillation 2, its pointwise OOD accuracy is $91.24\pm9.41\%$.

\subsection{Baseline configurations and RESTART reproduction}
\label{app:baselines}
For LLMDirect, LLM-SR, LaSR, SGA, and DE, Table~\ref{tab:main} uses the Llama3.1-8B-Instruct results reported by DE \citep{de}. We preserve DE's reported configurations rather than rerunning these methods.

\noindent\textbf{Reported baseline parameters.}
The imported runs share the following settings:
\begin{itemize}
    \item \textbf{Backbone:} Llama3.1-8B-Instruct.
    \item \textbf{Generation temperature:} 0.8.
    \item \textbf{Maximum output length:} 8,192 tokens.
    \item \textbf{Inference system:} vLLM on 4 NVIDIA A800-SXM4-80GB GPUs.
\end{itemize}

\noindent\textbf{LLMDirect.}
\begin{itemize}
    \item \textbf{Hypotheses per prompt:} 5.
    \item \textbf{Total LLM-generated candidates per problem:} 5.
    \item \textbf{Execution timeout per hypothesis:} 30 seconds.
    \item \textbf{Constant optimizer:} SciPy BFGS.
\end{itemize}

\noindent\textbf{LLM-SR.}
\begin{itemize}
    \item \textbf{Equation programs per prompt:} 4.
    \item \textbf{Parallel evaluators:} 4.
    \item \textbf{Islands:} 10.
    \item \textbf{In-context parents per prompt:} 4.
    \item \textbf{Maximum parameters per equation skeleton:} 10.
    \item \textbf{Constant optimizer:} SciPy BFGS.
    \item \textbf{Total LLM-generated candidates per problem:} 1,000.
\end{itemize}

\Needspace{4\baselineskip}
\noindent\textbf{LaSR.}
\begin{itemize}
    \item \textbf{Iterations:} 25.
    \item \textbf{Evolutionary cycles per iteration:} 550.
    \item \textbf{Populations:} 10.
    \item \textbf{Population size:} 33.
    \item \textbf{Operators:} $+,-,\times,\div,\wedge,\exp,\log,\sqrt{\cdot},\sin,\cos,\tan,$ and $\cosh$.
    \item \textbf{LLM-operator weights:} \texttt{llm\_mutate}=0.005, \texttt{llm\_crossover}=0.005, and \texttt{llm\_gen\_random}=0.005.
    \item \textbf{Total LLM-generated candidates per problem:} 1,000.
    \item \textbf{Random non-LLM mutations per problem:} $>453{,}000$.
    \item \textbf{Other PySR parameters:} defaults of the evaluated implementation; DE reports no further values.
\end{itemize}

\noindent\textbf{SGA.}
\begin{itemize}
    \item \textbf{Feedback objective:} MSE.
    \item \textbf{Continuous-parameter optimizer:} \texttt{torch.optim.Adam}.
    \item \textbf{Differentiable model:} PyTorch \texttt{torch.nn.Module}.
    \item \textbf{Total LLM-generated candidates per problem:} 1,000.
    \item \textbf{Adam hyperparameters:} not reported by DE.
\end{itemize}

\noindent\textbf{DE.}
\begin{itemize}
    \item \textbf{Islands:} 4; \textbf{capacity per island:} 400; \textbf{reset interval:} 50 generations.
    \item \textbf{Generations:} 100; \textbf{offspring per generation:} 4; \textbf{maximum LLM-generated candidates per problem:} 400.
    \item \textbf{Selection:} fitness-proportional Boltzmann sampling; \textbf{initial temperature:} 0.5; \textbf{cooling rate:} 0.95; \textbf{schedule:} exponential decay; \textbf{stagnation-recovery temperature:} 2.0.
    \item \textbf{Parents per prompt:} 1 for non-crossover operators and 2 for crossover.
    \item \textbf{Maximum refinement rounds per sample:} 4.
    \item \textbf{Reflective-memory update interval:} 12 generations; \textbf{breakthrough threshold:} 0.1.
    \item \textbf{Reflection context:} 3 elite examples and 3 failed examples.
    \item \textbf{Stagnation window:} 8 generations; \textbf{distilled insights per update:} 3.
\end{itemize}

The RESTART row in Table~\ref{tab:main} uses our reproduction, rather than the original paper's reported scores. It evaluates all 240 equations with three seeds (720 valid task results).

\begin{table}[!htbp]
\caption{Settings for our RESTART reproduction.}
\label{tab:restart_config}
\centering
\small
\setlength{\tabcolsep}{4pt}
\begin{tabular}{@{}ll@{}}
\toprule
Setting & Value\\
\midrule
Backbone / temperature & Llama3.1-8B-Instruct / 0.8\\
Training / internal-validation fractions & 0.8 / 0.2\\
Islands / generations / slots per generation & 10 / 100 / 4\\
LLM-generated candidate budget & 400 slots per problem\\
Refinement and memory & E2E-assisted refinement / concept library\\
Maximum fit rows / constants & 4,096 / 10\\
BFGS starts / initial constant & 3 / 1.0\\
BFGS iterations per parameter / gradient tolerance & 200 / $10^{-5}$\\
Per-slot fitting timeout / objective evaluations & 30 seconds / 6,000\\
Parameter domain & Unbounded real values\\
Final selection & Minimum internal-validation MSE\\
Additional computation & Numerical/E2E initialization and E2E calls\\
\bottomrule
\end{tabular}
\end{table}

The reproduction adapts the released island, refinement, and concept-library code to the same numerical-fitting interface used by \ours{}. The 10 islands, E2E-based refinement, and concept library retain the method's native roles. Numerical/E2E initialization uses 0 LLM proposal slots, and E2E calls add computation beyond the 400-slot budget.

Generated Python candidates are converted to the shared restricted algebra without executing generated code. The converter preserves the generated formula structure. Fitting uses the shared BFGS backend and budgets; our post-fit expression-pruning search is not applied to RESTART candidates. Final selection uses minimum internal-validation MSE. The official test arrays are scored only after search completion, without refitting.

\begin{table}[!htbp]
\caption{Per-problem search-budget comparison. Candidate or proposal counts follow the accounting used by DE and our implementations; the final column records method-specific computation outside those counts.}
\label{tab:budget_comparison}
\centering
\footnotesize
\setlength{\tabcolsep}{3pt}
\begin{tabular}{@{}>{\raggedright\arraybackslash}p{0.14\linewidth}>{\raggedright\arraybackslash}p{0.22\linewidth}>{\raggedright\arraybackslash}p{0.56\linewidth}@{}}
\toprule
Method & Counted LLM search budget & Additional search or optimization\\
\midrule
LLMDirect & 5 generated candidates & BFGS constant refinement for each hypothesis.\\
LLM-SR & 1,000 generated candidates & Island evolution and BFGS constant refinement; these operations act on the counted candidates.\\
LaSR & 1,000 LLM-generated candidates & More than 453,000 random non-LLM mutations, in addition to PySR fitting and selection.\\
SGA & 1,000 generated candidates & Adam-based inner-loop optimization of continuous parameters.\\
DE & At most 400 generated candidates ($100\times4$) & Diagnostic tools, BFGS constant fitting, and reflective-memory updates are outside the candidate count.\\
RESTART (ours) & 400 proposal slots ($100\times4$) & Numerical/E2E initialization and E2E calls do not consume proposal slots; constant fitting follows the shared BFGS limits.\\
\ours{} & 400 charged proposal slots ($40+90\times4$) & \\
\bottomrule
\end{tabular}
\end{table}

LLMDirect is a Best-of-5 baseline. DE, our RESTART reproduction, and \ours{} use 40\% of the 1,000-candidate LLM budget assigned to the iterative LLM-SR, LaSR, and SGA baselines. LaSR additionally evaluates more than 453 non-LLM mutations per counted LLM candidate. RESTART and \ours{} match at 400 nominal proposal slots.

\subsection{Verifier ablation details}
\label{app:representation_ablation}
Each variant in Table~\ref{tab:representation_ablation} uses training seed 42, Stage~2 batch size 256, and 100,000 updates. Data-pool snapshots, the 12-bin sampler, the loss in Equation~\ref{eq:loss}, validation composition, and optimization settings match the full verifier. The fixed validation set contains 173 edited and 211 random pairs, with 32 pairs per TED bin. Among its 384 pairs, 327 have $\dstr\leq1$; the near-target MAE and RMSE use only those 327 pairs. Global MAE uses all pairs.

Without tree address, the formula encoder still receives prefix tokens, but its depth/address embedding contribution is zero in both training stages. The original position parameters remain allocated, giving the same total parameter count as the full model, although the zero-gated parameters have no effect. Formula only reuses the full model's Stage~1 formula encoder, retains its tree-address inputs, and trains only that encoder and the fusion head. Its numerical encoders are absent, so its parameter count is smaller. Without the correlation branch, the data and formula encoders are initialized from the same full-model Stage~1 checkpoint, but the signed-correlation encoder is absent and its fusion slot is fixed to zero. One-stage retains all full-model inputs and parameters, uses no Stage~1 checkpoint, and trains from random initialization for the same 100,000 updates under the Stage~2 objective.

We select all reported checkpoints by the same validation objective.

\end{document}